\def\year{2021}\relax
\documentclass[letterpaper]{article} %
\usepackage{aaai21} %
\usepackage{times} %
\usepackage{helvet} %
\usepackage{courier} %
\usepackage[hyphens]{url} %
\usepackage{graphicx} %
\usepackage{graphicx} %
\usepackage{natbib} %
\usepackage{caption} %
\usepackage{tikz}
	\usetikzlibrary{positioning,fit,calc, decorations, arrows, shapes, shapes.geometric}
	\usetikzlibrary{backgrounds}
	\usetikzlibrary{patterns}
	\usetikzlibrary{cd}
	
	\pgfdeclaredecoration{arrows}{draw}{
		\state{draw}[width=\pgfdecoratedinputsegmentlength]{%
			\path [every arrow subpath/.try] \pgfextra{%
				\pgfpathmoveto{\pgfpointdecoratedinputsegmentfirst}%
				\pgfpathlineto{\pgfpointdecoratedinputsegmentlast}%
			};
	}}
	\tikzset{AmpRep/.style={ampersand replacement=\&}}
	\tikzset{center base/.style={baseline={([yshift=-.8ex]current bounding box.center)}}}
	\tikzset{paperfig/.style={center base,scale=0.9, every node/.style={transform shape}}}

	\tikzset{dpadded/.style={rounded corners=2, inner sep=0.7em, draw, outer sep=0.3em, fill={black!50}, fill opacity=0.08, text opacity=1}}
	\tikzset{dpad0/.style={outer sep=0.05em, inner sep=0.3em, draw=gray!75, rounded corners=4, fill=black!08, fill opacity=1}}
	\tikzset{dpad/.style args={#1}{every matrix/.append style={nodes={dpadded, #1}}}}
	\tikzset{light pad/.style={outer sep=0.2em, inner sep=0.5em, draw=gray!50}}
		
	\tikzset{arr/.style={draw, ->, thick, shorten <=3pt, shorten >=3pt}}
	\tikzset{arr1/.style={draw, ->, thick, shorten <=1pt, shorten >=1pt}}
	\tikzset{arr2/.style={draw, ->, thick, shorten <=2pt, shorten >=2pt}}

	\tikzset{fgnode/.style={dpadded,inner sep=0.6em, circle},
	factor/.style={light pad, fill=black}}

	\newcommand\cmergearr[4]{
		\draw[arr,-] (#1) -- (#4) -- (#2);
		\draw[arr, shorten <=0] (#4) -- (#3);
	}
	\newcommand\mergearr[3]{
		\coordinate (center-#1#2#3) at (barycentric cs:#1=1,#2=1,#3=1.2);
		\cmergearr{#1}{#2}{#3}{center-#1#2#3}
	}

	\usetikzlibrary{matrix}
	\tikzset{toprule/.style={%
	        execute at end cell={%
	            \draw [line cap=rect,#1] 
	            (\tikzmatrixname-\the\pgfmatrixcurrentrow-\the\pgfmatrixcurrentcolumn.north west) -- (\tikzmatrixname-\the\pgfmatrixcurrentrow-\the\pgfmatrixcurrentcolumn.north east);%
	        }
	    },
	    bottomrule/.style={%
	        execute at end cell={%
	            \draw [line cap=rect,#1] (\tikzmatrixname-\the\pgfmatrixcurrentrow-\the\pgfmatrixcurrentcolumn.south west) -- (\tikzmatrixname-\the\pgfmatrixcurrentrow-\the\pgfmatrixcurrentcolumn.south east);%
	        }
	    },
	    leftrule/.style={%
	        execute at end cell={%
	            \draw [line cap=rect,#1] (\tikzmatrixname-\the\pgfmatrixcurrentrow-\the\pgfmatrixcurrentcolumn.north west) -- (\tikzmatrixname-\the\pgfmatrixcurrentrow-\the\pgfmatrixcurrentcolumn.south west);%
	        }
	    },
	    rightrule/.style={%
	        execute at end cell={%
	            \draw [line cap=rect,#1] (\tikzmatrixname-\the\pgfmatrixcurrentrow-\the\pgfmatrixcurrentcolumn.north east) -- (\tikzmatrixname-\the\pgfmatrixcurrentrow-\the\pgfmatrixcurrentcolumn.south east);%
	        }
	    },
	    table with head/.style={
		    matrix of nodes,
		    row sep=-\pgflinewidth,
		    column sep=-\pgflinewidth,
		    nodes={rectangle,minimum width=2.5em, outer sep=0pt},
		    row 1/.style={toprule=thick, bottomrule},
  	    }
	}
\newif\ifprecompiledfigs
\precompiledfigstrue

\newif\ifexternalizefigures\externalizefiguresfalse

\ifexternalizefigures
	\usetikzlibrary{external}
\fi

\usepackage{booktabs}       %
\usepackage{amsfonts}       %
\usepackage{nicefrac}       %
\usepackage{microtype}      %
\usepackage{mathtools}		%
\usepackage{amssymb, bbm}
\usepackage{relsize}
\usepackage{environ} %

\usepackage{color}

\usepackage{amsthm}
\usepackage{thmtools}
\usepackage{xstring}
\usepackage{enumitem}

\theoremstyle{plain}
\newtheorem{theorem}{Theorem}[section]
\newtheorem{coro}{Corollary}[theorem]
\newtheorem{prop}[theorem]{Proposition}
\newtheorem{lemma}[theorem]{Lemma}
\newtheorem{fact}[theorem]{Fact}

\theoremstyle{definition}
\declaretheorem[name=Definition,qed=$\square$,numberwithin=section]{defn} %
\declaretheorem[name=Construction,qed=$\square$,sibling=defn]{constr}
\declaretheorem[qed=$\square$]{example}

\theoremstyle{remark}

\def\wrapwith#1[#2;#3]{
	\expandarg\IfSubStr{#1}{,}{
		\expandafter#2{\expandarg\StrBefore{#1}{,}}
		\expandarg\StrBehind{#1}{,}[\tmp] 
		\xdef\tmp{\expandafter\unexpanded\expandafter{\tmp}}
		#3
		\wrapwith{\tmp}[#2;{#3}]
	}{ \expandafter#2{#1} }
}
\def\hwrapcells#1[#2]{\wrapwith#1[#2;&]}
\def\vwrapcells#1[#2]{\wrapwith#1[#2;\\]}
\NewEnviron{mymathenv}{$\BODY$}

\newcommand{\smalltext}[1]{\text{\footnotesize#1}}
\newsavebox{\idxmatsavebox}
\def\makeinvisibleidxstyle#1#2{\phantom{\hbox{#1#2}}}
\newenvironment{idxmatphant}[4][\color{gray}\smalltext]{%
	\def\idxstyle{#1}
	\def\colitems{#3}
	\def\rowitems{#2}
	\def\phantitems{#4}
	\begin{lrbox}{\idxmatsavebox}$%$\begin{mymathenv}
	\begin{matrix}  \begin{matrix} \hwrapcells{\colitems}[\idxstyle]  \end{matrix} \\[0.1em]
		\left[ 
		\begin{matrix}
			\hwrapcells{\phantitems}[\expandafter\makeinvisibleidxstyle\idxstyle]  \\[-1em]
	}{
		\end{matrix}\right]		&\hspace{-0.5em}\begin{matrix*}[l] \vwrapcells{\rowitems}[\idxstyle] \end{matrix*}
	\end{matrix}%
	$%\end{mymathenv}
	\end{lrbox}%
	\raisebox{0.75em}{\usebox\idxmatsavebox}
}

\newenvironment{idxmat}[3][\color{gray}\smalltext]
	{\begingroup\idxmatphant[#1]{#2}{#3}{#3}}
	{\endidxmatphant\endgroup}

\makeatletter

\def\env@cases#1{%
	\let\@ifnextchar\new@ifnextchar
	\left\lbrace\def\arraystretch{1.2}%
	\array{@{}#1@{\quad}l@{}}}
\makeatother

\usepackage[noabbrev,nameinlink,capitalize]{cleveref}
\crefname{fact}{Fact}{Facts}
\crefname{example}{Example}{Examples}
\crefname{defn}{Definition}{Definitions}
\crefname{prop}{Proposition}{Propositions}
\crefname{constr}{Construction}{Constructions}

\usepackage{float}
\newcounter{subfigure}

\newcommand{\restate}[2]
	{\medskip\par\noindent{\bf \expandarg\Cref{thmt@@#1}.}%
 	\noindent\begingroup\em #2 \endgroup\par\smallskip}

\newcommand{\recall}[1]{\medskip\par\noindent{\bf \expandarg\Cref{thmt@@#1}.} \begingroup\em \noindent
   \expandafter\csname#1\endcsname* \endgroup\par\smallskip}

\usepackage{xpatch}
\makeatletter
\xpatchcmd{\thmt@restatable}%
   {\csname #2\@xa\endcsname\ifx\@nx#1\@nx\else[{#1}]\fi}%
   {\ifthmt@thisistheone\csname #2\@xa\endcsname\ifx\@nx#1\@nx\else[{#1}]\fi
   \else\fi}
   {\typeout{oii Success1?}}{\typeout{oiii failure1?}} %
\xpatchcmd{\thmt@restatable}%
   {\csname end#2\endcsname}
   {\ifthmt@thisistheone\csname end#2\endcsname\else\fi}
   {\typeout{oii Success2?}}{\typeout{oiii failure2?}}
\makeatother

\allowdisplaybreaks
\newcommand{\begthm}[3][]{\begin{#2}[{name=#1},restate=#3,label=#3]}

\newif\ifappendix
\appendixtrue
\newif\ifbody\bodytrue

\DeclarePairedDelimiterX{\bbr}[1]{[}{]}{\mspace{-3.5mu}\delimsize[#1\delimsize]\mspace{-3.5mu}}
\DeclarePairedDelimiterXPP{\SD}[1]{}{[}{]}{_{\text{sd}}}{\mspace{-3.5mu}\delimsize[#1\delimsize]\mspace{-3.5mu}}

\let\H\relax
\DeclareMathOperator{\H}{\mathrm{H}} %
\DeclareMathOperator{\I}{\mathrm{I}} %
\DeclareMathOperator*{\Ex}{\mathbb{E}} %
\DeclareMathOperator*{\argmin}{arg\;min}
\newcommand{\CI}{\mathrel{\perp\mspace{-10mu}\perp}} %
\newcommand\mat[1]{\mathbf{#1}}
\newcommand\Pa{\mathbf{Pa}}
\DeclarePairedDelimiterX{\infdivx}[2]{(}{)}{%
	#1\;\delimsize\|\;#2%
}
\newcommand{\thickD}{I\mkern-8muD}
\newcommand{\kldiv}{\thickD\infdivx}

\newcommand{\tto}{\rightarrow\mathrel{\mspace{-15mu}}\rightarrow}
\def\sheq{\!=\!}

\newcommand{\bp}[1][L]{\mat{p}_{\!_{#1}\!}}
\newcommand{\V}{\mathcal V}
\newcommand{\N}{\mathcal N}
\newcommand{\Ed}{\mathcal E}
\newcommand{\pdgvars}[1][]{(\N#1, \Ed#1, \V#1, \mat p#1, \beta#1)}

\DeclareMathAlphabet{\mathdcal}{U}{dutchcal}{m}{n}
\DeclareMathAlphabet{\mathbdcal}{U}{dutchcal}{b}{n}

\newcommand{\dg}[1]{\mathbdcal{#1}}
\newcommand{\pdgunit}{\mathrlap{\mathit 1} \mspace{2.3mu}\mathit 1}

\newcommand{\IDef}[1]{\mathit{IDef}_{\!#1}}
\newcommand\Inc{\mathit{Inc}}
\newcommand{\PDGof}[1]{{\dg M}_{#1}}
\newcommand{\UPDGof}[1]{{\dg N}_{#1}}
\newcommand{\WFGof}[1]{\Psi_{{#1}}}
\newcommand{\FGof}[1]{\Phi_{{#1}}}
\newcommand{\Gr}{\mathcal G}
\newcommand\GFE{\mathit{G\mkern-4mu F\mkern-4.5mu E}}

\newcommand{\ed}[3]{#2\!%
  \overset{\smash{\mskip-5mu\raisebox{-1pt}{$\scriptscriptstyle
        #1$}}}{\rightarrow}\! #3} 
\newcommand{\alle}[1][L]{_{ \ed {#1}XY}}

\newcommand{\notation}[1]{\ignorespaces} % disable notation
\newcommand\vfull[1]{} % disable vfull

\title{Probabilistic Dependency Graphs}
\author{
Oliver Richardson, Joseph Y. Halpern \\
}
\affiliations {
	Computer Science Department \\
	Cornell University \\
Ithaca, NY 14853 \\[0.05em]
	\{oli, halpern\}@cs.cornell.edu
}

\begin{document}
\ifbody
\maketitle
\begin{abstract}
We introduce Probabilistic Dependency Graphs (PDGs), a new class of
directed graphical models.   PDGs can capture inconsistent beliefs in a
natural way and are more modular than Bayesian Networks (BNs), in that
they make it easier to incorporate new information and restructure the  
representation.    We show by example how PDGs are an especially natural
modeling tool.
We provide three semantics for PDGs, each of which can be derived from a
scoring function (on joint distributions over the
variables in the network) that can be viewed as representing a
distribution's incompatibility with the PDG.
For the PDG corresponding
to a BN, this function is uniquely minimized by the distribution the
BN represents, showing that PDG semantics extend BN semantics.  
We show further that factor graphs
and their exponential families
can also be faithfully represented as PDGs%
, while there are significant barriers to modeling a PDG with a factor graph.
\end{abstract}

\section{Introduction}

In this paper we introduce yet another graphical
tool
for modeling beliefs,
\emph{Probabilistic Dependency Graphs} (PDGs). There are already many
such models in the literature, including Bayesian networks (BNs) and
factor graphs. (For an overview, see 
\cite{KF09}.)
Why does the world need one more?  

Our original motivation for introducing PDGs was to be able capture
inconsistency. We want to be able to model the process of resolving
inconsistency; to do so, we have to model the inconsistency itself. But our
approach to modeling inconsistency has many other advantages. In particular,
PDGs are significantly more modular than other directed graphical models:
operations like restriction and union that are easily done with PDGs are
difficult or impossible to do with other representations.
The following examples motivate PDGs and illustrate some of
their advantages.

\begin{example} \label{ex:guns-and-floomps}
Grok is visiting a neighboring district. From prior reading, she thinks it
likely (probability .95) that guns are illegal here. Some brief conversations
with locals lead her to believe that, with probility .1, the law
prohibits floomps.

The obvious way to represent this as a BN is to use two random variables
$F$ and $G$ (respectively taking values $\{f, \smash{\overline f}\}$ and $g,
\overline g$),
indicating whether  floomps and guns are prohibited.
The semantics of a BN offer her two choices: either assume that $F$ and $G$
to be independent and give (unconditional) probabilities of $F$ and $G$, or
choose a direction of dependency, and give one of the two unconditional
probabilities and a conditional probability distribution. 
As there is no reason to choose either direction of dependence, the
natural choice is to 
assume independence, giving her the 
BN on the left of \Cref{fig:gun-floomp-diagram}.

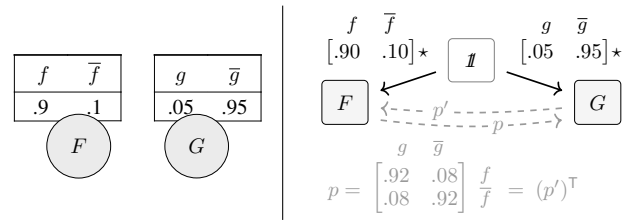
\begin{figure}[htb]
  \centering
\ifprecompiledfigs
	\raisebox{-0.5\height}{\includegraphics[scale=0.8]{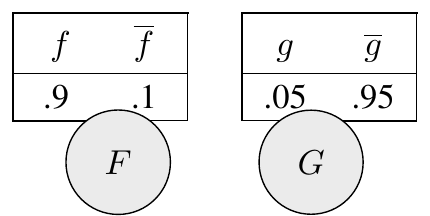}}
~\vrule~
	\raisebox{-0.5\height}{\includegraphics[scale=0.8]{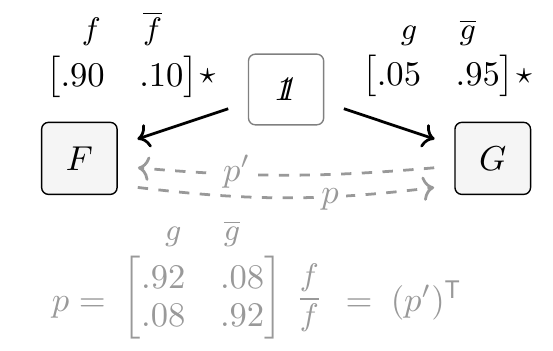}}
\else
	\scalebox{0.8}{
\begin{tikzpicture}[center base, scale=0.7, AmpRep]
        \def\figtabledist{0.2}
        \def\fignodedist{1.4}
        \def\figtableheight{0.41} 

        \matrix [table with head, column 1/.style={leftrule}, anchor=south east,
             column 2/.style={rightrule}, row 2/.style={bottomrule}] at (-\figtabledist,\figtableheight) {
            \vphantom{$\overline fg$} $f$ \& \vphantom{$\overline fg$}$\overline f$\\
            .9 \& .1\\
        };
        \matrix [table with head, column 1/.style={leftrule}, anchor=south west,
             column 2/.style={rightrule}, row 2/.style={bottomrule}] at (\figtabledist,\figtableheight) {
             \vphantom{$\overline fg$}$g$ \& \vphantom{$\overline fg$}$\overline g$\\
             .05 \& .95\\
        };
        \node[dpadded, circle, fill=black!08, fill opacity=1] (floomp) at (-\fignodedist,0) {$F$};
        \node[dpadded, circle, fill=black!08, fill opacity=1] (gun) at (\fignodedist,0) {$G$};
    \end{tikzpicture}
    ~~\vrule~~
	\begin{tikzpicture}[center base]

        \def\fignodedist{2.1}
        \def\fignodeheight{1.1}
        \def\newcptX{-0.3}
        \def\newcptY{-0.1}
                     
		\node[dpadded, fill=white, draw=gray] (true)  at (0,1.8) {$\pdgunit$};
		\node[dpadded] (floomp) at (-\fignodedist,\fignodeheight) {$F$};
		\node[dpadded] (gun) at (\fignodedist,\fignodeheight) {$G$};			
		
		\draw[arr] (true) to[bend left=0] coordinate(A) (floomp);
		\draw[arr] (true) to[bend right=0] coordinate(B) (gun);

		\node[above left=2.0em and 1.5em of A, anchor=center] {
			\begin{idxmat}[\color{black}\smalltext]{\!\!\!$\star$\;\;\;}{$f$, $\overline f$}
				.90 & .10 \\
			\end{idxmat}
		};
		\node[above right=2.0em and 1.3em of B, anchor=center] {
			\begin{idxmat}[\color{black}\smalltext]{\!\!\!$\star$}{$g$, $\overline g$}
				.05 & .95 \\
			\end{idxmat}
		};
		\definecolor{heldout}{rgb}{0.6, 0.6, .6}	
		\draw[heldout, dashed, arr] (floomp.-30) to[bend right=7] node[pos=0.65, fill=white, inner sep=2pt] (C) {$\smash{p}\vphantom{v}$} (gun.210);
        \draw[heldout, dashed, arr] (gun.190) to[bend left=5] node[pos=0.668, fill=white, inner sep=2pt] {$\smash{p'}\vphantom{v}$} (floomp.-10);
		\node[anchor=center] (newcpd) at (\newcptX,\newcptY) {
			\color{heldout}
		$p
			=\!\!\!$\begin{idxmat}[\color{heldout}\smalltext]{$f$,$\overline f$}{$g$, $\overline g$}
			  .92 & .08 \\ .08 & .92 \\
          \end{idxmat}$~=~
		    {(p')^{\textsf T}}$
		};
	\end{tikzpicture}
	}
\fi
        \caption{A BN (left) and corresponding PDG (right), which can
        include more cpds; $p$ or $p'$ make it inconsistent.} 
    \label{fig:gun-floomp-diagram}
\end{figure}

A traumatic experience a few hours later leaves Grok believing that
``floomp'' is likely (probability .92) to be another word for gun.
Let $p(G \mid F)$ be the \emph conditional \emph probability \emph
distribution (cpd) that describes 
the belief that if floomps are legal (resp., illegal),
then with probability .92, guns are as well, and $p'(F \mid G)$ be
the reverse. 
Starting with $p$, Grok's first instinct is to
simply incorporate the conditional information by adding $F$ as a parent of
$G$, and then associating
the cpd
$p$ with $G$. But then what should she do
with the original probability she had for $G$?  Should she just discard it?
It is easy to check that there is no 
joint distribution
that is consistent with
both
the two original priors on $F$ and $G$ and also 
$p$.  So if she
is to represent the information with a BN, which always represents a consistent
distribution, she must resolve the inconsistency.

However,
sorting this out immediately may not be ideal.
For instance, if the inconsistency arises from a conflation between
two definitions 
of ``gun'', a resolution will have destroyed the original cpds. A
better use of computation may be to notice the inconsistency and look
up the actual law. 

By way of contrast, consider the corresponding PDG. In a PDG, the cpds are
attached to edges, rather than nodes of the graph.
In order to represent unconditional probabilities, we introduce
a \emph{unit variable} $\pdgunit$ which 
takes only one value, denoted
$\star$. 
This leads Grok to 
the PDG depicted in \Cref{fig:gun-floomp-diagram},
where the edges from $\pdgunit$ to $F$ and $G$ are associated with the
unconditional probabilities of $F$ and $G$, and the 
edges between $F$ and $G$ are associated with $p$ and $p'$.

The original state of knowledge consists of all three nodes and the two
solid
edges from $\pdgunit$. This is like Bayes Net that we considered above,
except that we 
no longer
explicitly
take  $F$ and $G$ to be independent; we merely record the constraints
imposed by the given probabilities.  
	
The key point is that we can incorporate the new information into our original
representation (the graph in \Cref{fig:gun-floomp-diagram} without the edge from
$F$ to $G$) simply  by adding the edge from $F$ to $G$ and the associated cpd
$p$ (the new infromation is shown in blue).
Doing so does not change the meaning of the original edges.   
Unlike a Bayesian update, the operation is even reversible: all we need
to do recover our original belief state is delete the new edge, 
making it possible to mull over and then reject an observation.
\end{example}

The ability of PDGs to model inconsistency, as illustrated in
\Cref{ex:guns-and-floomps}, appears to have come at a significant cost. We seem
to have lost a key benefit of BNs: the ease with which they can
capture
(conditional) independencies, which, as Pearl (\citeyear{pearl}) has
argued forcefully, are omnipresent.

\begin{example}[emulating a BN]\label{ex:smoking}

We now consider the classic (quantitative) Bayesian network $\cal B$, which has
four binary variables indicating whether a person ($C$) develops cancer, ($S$)
smokes, ($\mathit{SH}$) is exposed to second-hand smoke, and ($\mathit{PS}$) has
parents who smoke, presented graphically in \Cref{subfig:smoking-bn}. We now
walk through what is required to represent $\cal B$ as a PDG, which we call
$\PDGof{{\mathcal B}}$, shown as the solid nodes and edges in
\Cref{subfig:smoking-pdg}.

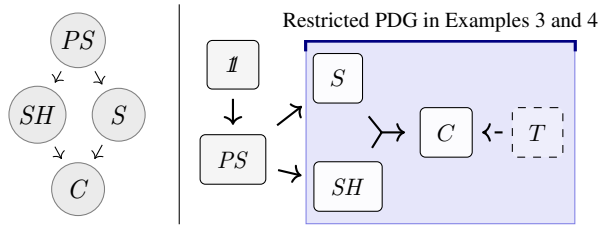
\begin{figure}[ht!]
\addtocounter{figure}{1}
\centering
\hfill
\ifprecompiledfigs
\raisebox{-0.5\height}{\includegraphics{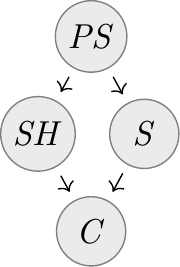}}
\else
\begin{tikzpicture}[paperfig]
	\begin{scope}[every node/.style={dpadded, fill opacity=1,fill=black!08, circle, inner sep=2pt, minimum size=2em, draw=gray}]
		\node (PS) at (0,1.1) {$\mathit{PS}$};
		\node (SH) at (-0.6,0) {$\mathit{SH}$};
		\node (S) at (0.6,0) {$\mathit{S}$};
		\node (C) at (0,-1.1) {$\mathit{C}$};
	\end{scope}
	\draw[->] (PS) to (S);
	\draw[->] (PS) to (SH);
	\draw[->] (SH) to (C);
	\draw[->] (S) to (C);
\end{tikzpicture}
\fi
\refstepcounter{subfigure}
\label{subfig:smoking-bn}
~~\vline~~
\ifprecompiledfigs
\raisebox{-0.5\height}{\includegraphics{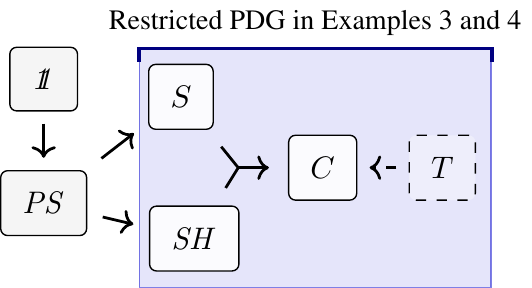}}
\else
\begin{tikzpicture}[paperfig]
	\fill[fill opacity=0.1, blue!80!black, draw, draw opacity=0.5] (2.73,1.35) rectangle (6.7, -1.35);
	
	\node[dpadded] (1) at (1.65,1) {$\pdgunit$};
	\node[dpadded] (PS) at (1.65,-0.4) {$\mathit{PS}$};
	\node[dpadded, fill=black!.16, fill opacity=0.9] (S) at (3.2, 0.8) {$S$};
	\node[dpadded, fill=black!.16, fill opacity=0.9] (SH) at (3.35, -0.8) {$\mathit{SH}$};
	\node[dpadded, fill=black!.16, fill opacity=0.9] (C) at (4.8,0) {$C$};
	
	\draw[arr1] (1) -- (PS);
	\draw[arr2] (PS) -- (S);
	\draw[arr2] (PS) -- (SH);
	\mergearr{SH}{S}{C}
	
	\node[dpadded, fill=black!.16, fill opacity=0.35, dashed] (T) at (6.15,0) {$T$};
	\draw[arr1,dashed] (T) -- (C);	

	\draw[very thick, |-|, color=blue!50!black,text=black] (2.7, 1.35) --coordinate(Q) (6.73,1.35);%
	\fill[white] (2.6, 1.36) rectangle (6.9,1.55);
	\node[above=0.05em of Q]{\small Restricted PDG in \cref{ex:grok-ablate,ex:grok-union}};
\end{tikzpicture}
\fi
	\hfill~
	\refstepcounter{subfigure}
	\label{subfig:smoking-pdg}
\addtocounter{figure}{-1}
\caption{ (a) The Bayesian Network $\mathcal B$ in \cref{ex:smoking} (left), and
(b) $\PDGof{\mathcal B}$, its corresponding PDG (right). The shaded box
indicates a restriction of $\PDGof{\mathcal B}$ to only the nodes and edges it
contains, and the dashed node $T$ and its arrow to $C$ can be added in the PDG,
without taking into account $S$ and $SH$.}
\label{fig:smoking-bn+pdg}
\end{figure}

We start with the nodes corresponding to the variables in $\cal B$, together
with the special node $\pdgunit$ from \Cref{ex:guns-and-floomps}; we add an edge
from ${\pdgunit}$ to $\mathit{PS}$, to which we associate the unconditional
probability given by the cpd for $\mathit{PS}$ in $\cal B$. We can also re-use
the cpds for $S$ and $\mathit{SH}$, assigning them, respectively, to the edges
$PS \to S$ and $PS \to SH$ in $\PDGof{{\mathcal B}}$.
There are two remaining problems: (1) modeling the remaining table in $\cal B$,
which corresponds to the conditional probability of $C$ given $S$ and $SH$; and
(2) recovering the additional
conditional
independence assumptions in the BN. 

For (1), we cannot just add the edges $S \to C$ and $SH \to C$ that are present
in $\cal B$. As we saw in \Cref{ex:guns-and-floomps}, this would mean
supplying two \emph{separate} tables, one indicating the probability of $C$
given $S$, and the other indicating the probability of $C$ given
$\mathit{SH}$.  We would lose significant information that is
present in $\cal B$  about 
how $C$ depends jointly on $S$ and $SH$. To distinguish the joint dependence on
$S$ and $\mathit{SH}$, for now, we draw an edge with two tails---a
(directed)
\emph{hyperedge}---that completes the diagram in \Cref{subfig:smoking-pdg}. 
With regard to (2), there are many distributions consistent with the conditional
marginal probabilities in the cpds, and the independences presumed by $\cal B$
need not hold for them. 
Rather than trying to distinguish between them with additional constraints,
we develop a a scoring-function semantics for PDGs
which 
is in this case uniquely minimized by the distribution 
specified by ${\mathcal B}$ (\Cref{thm:bns-are-pdgs}).
This allows us to recover the semantics of Bayesian networks without requiring the independencies that they assume.

Next suppose that we get information beyond that captured by the original BN.
Specifically, we read a thorough empirical study demonstrating that people who
use tanning beds have a 10\% incidence of cancer, compared with 1\% in the
control group 
(call the cpd for this $p$); we would like to add this information to
$\cal B$. The first step is clearly to add a new node labeled $T$, for ``tanning
bed use''.  But simply making $T$ a parent of $C$ (as clearly seems appropriate,
given that the incidence of cancer depends on tanning bed use) requires a
substantial expansion of the cpd; in particular, it requires us to make
assumptions about the interactions between tanning beds and smoking.  
The corresponding PDG, $\PDGof{{\mathcal B}}$, on the other hand, has no
trouble: We can simply add the node $T$ with an edge to $C$ that is associated
with $\mat p$.  But note that doing this makes it possible for our knowledge to
be inconsistent. To take a simple example, if the distribution on $C$ given $S$
and $H$ encoded in the original cpd was always deterministically ``has cancer''
for every possible value of $S$ and $H$, but the distribution according to the
new cpd from $T$ was deterministically ``no cancer'', the resulting PDG would be
inconsistent.  
\end{example}

We have seen that we can easily add information to PDGs; removing information is
equally painless.   

\begin{example}[restriction]\label{ex:grok-ablate}
  After the Communist party came to power,
  children were raised communally, and so parents' smoking habits no longer had any impact on them. Grok is reading her favorite book on graphical models, and she realizes that while the node $\mathit{PS}$ in \Cref{subfig:smoking-bn} has lost its usefulness, and nodes $S$ and $\mathit{SH}$ no longer ought to have $\mathit{PS}$ as a parent, the other half of the diagram---that is, the node $C$ and its dependence on $S$ and $\mathit{SH}$---should apply as before.
Grok has identified two obstacles to modeling deletion of information from a BN
by simply deleting nodes and their associated cpds. First, this restricted model
is technically no longer a BN (which in this case would require unconditional
distributions on $S$ and $\mathit{SH}$), but rather a \emph{conditional} BN
\cite{KF09}, which allows for these nodes to be marked as observations;
observation nodes do not have associated beliefs. Second, even regarded as a
conditional BN, the result of deleting a node may introduce \emph{new}
independence information, incompatible with the original BN. For instance, by
deleting the node $B$ in a chain $A \rightarrow B \rightarrow C$, one concludes
that $A$ and $C$ are independent, a conclusion incompatible with the original BN
containing all three nodes.   
PDGs do not suffer from either problem.  We can easily delete the
nodes labeled 1 and $PS$ in \Cref{subfig:smoking-pdg} to get the
restricted PDG shown in the figure, which captures Grok's updated information.
The resulting PDG has no edges leading to $S$ or $\mathit{SH}$, and hence no
distributions specified on them; no special modeling distinction between
observation nodes and other nodes are required. Because PDGs do not directly
make independence assumptions, the information in this fragment is truly a
subset of the information in the whole PDG. 	
\end{example}

Being able to form a well-behaved local picture and restrict knowledge is
useful, but an even more compelling reason to use PDGs is their ability to
aggregate information. 
	
\begin{example}\label{ex:grok-union}
Grok dreams of becoming Supreme Leader ($\it SL$), and has come up with a plan.
She has noticed that people who use tanning beds have significantly more power
than those who don't. Unfortunately, her mom has always told her that tanning
beds cause cancer; specifically, that 15\% of people who use tanning beds get
it, compared to the baseline of 2\%. Call this cpd $q$. Grok thinks people will
make fun of her if she uses a tanning bed and gets cancer, making becoming
Supreme Leader impossible. This mental state is depicted as  a PDG on the left
of \Cref{fig:grok-combine}.

Grok is reading about graphical models because she vaguely remembers that the
variables in \Cref{ex:smoking} match the ones she already knows about. When she
finishes reading the statistics on smoking and the original study on tanning
beds (associated to a cpd $\mat p$ in \Cref{ex:smoking}), but before she has
time to reflect, we can represent her (conflicted) knowledge state as the union
of the two graphs, depicted graphically on the right of \Cref{fig:grok-combine}.

\begin{figure}
	\hfill
	\ifprecompiledfigs
\raisebox{-0.5\height}{\includegraphics{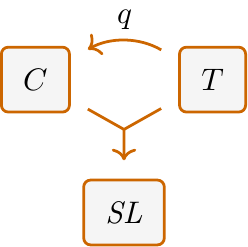}}
\hspace{1.2em}\vline\hspace{1.2em}
\raisebox{-0.5\height}{\includegraphics{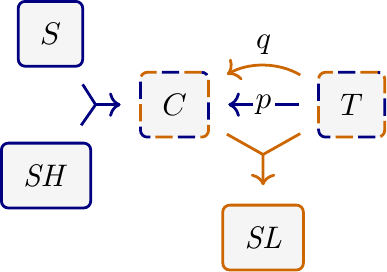}}
	\else
	\colorlet{colorsmoking}{blue!50!black}
	\colorlet{colororiginal}{orange!80!black}
	\tikzset{hybrid/.style={postaction={draw,colorsmoking,dash pattern= on 5pt off 8pt,dash phase=6.5pt,thick},
		draw=colororiginal,dash pattern= on 5pt off 8pt,thick}}
	\centering
	\begin{tikzpicture}[paperfig, thick, draw=colororiginal, text=black]
		\node[dpadded] (C) at (0,0) {$C$};
		\node[dpadded] (T) at (2,0){$T$};
		\node[dpadded] (SL) at (1,-1.5){$\it SL$};
		
		\draw[arr] (T) to[bend right] node[above]{$q$} (C);
		\mergearr{C}{T}{SL}
	\end{tikzpicture}
	\hspace{1.6em}\vline\hspace{1.6em}
	\begin{tikzpicture}[paperfig]
		\begin{scope}[postaction={draw,colorsmoking,dash pattern= on 3pt off 5pt,dash phase=4pt,thick}]
			
			\node[dpadded,hybrid] (C) at (0,0) {$C$};
			\node[dpadded,hybrid] (T) at (2,0){$T$};
		\end{scope}
		
		\begin{scope}[thick, draw=colororiginal, text=black]
			\node[dpadded] (SL) at (1,-1.5){$\it SL$};
			\draw[arr] (T) to[bend right] node[above]{$q$} (C);
			\mergearr{C}{T}{SL}
		\end{scope}

		\begin{scope}[thick, draw=colorsmoking, text=black]
			\node[dpadded] (S) at (-1.4, 0.8) {$S$};
			\node[dpadded] (SH) at (-1.45, -0.8) {$\mathit{SH}$};
			\draw[arr] (T) to node[fill=white, fill opacity=1,text opacity=1,inner sep=1pt]{$p$} (C);
			\mergearr{S}{SH}{C}
		\end{scope}
	\end{tikzpicture}
	\fi
	\hfill~
	\caption{Grok's prior (left) and combined (right) knowledge.}
	\label{fig:grok-combine}
\end{figure}
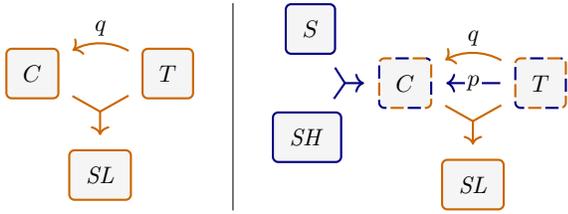

The union of the two PDGs, even with overlapping 
nodes, is still a PDG.
This is not the case in general
for BNs.
Note that the PDG that Grok used to
represent her two different sources of information (the mother's wisdom and the
study) regarding the distribution of $C$ is a \emph{multigraph}: there are two
edges from $T$ to $C$, with inconsistent information.
Had we not allowed multigraphs, we would have needed to choose between the two edges, or represent the
information some other (arguably less natural) way. As we are already allowing
inconsistency, merely recording both is much more in keeping with the way we
have handled other types of uncertainty. 
\end{example}

Not all inconsistencies are equally egregious. For example, even though the cpds
$p$ and $q$ are different, they are numerically close, so, intuitively, the PDG on the right in
\Cref{fig:grok-combine} is not very inconsistent.
Making this precise 
is
the focus of \Cref{sec:scoring-semantics}.

These examples give a taste of the power of PDGs.  In the coming sections, we formalize PDGs and relate them to other approaches.

\section{Syntax}\label{sec:formal+syntax}
We now provide formal definitions for PDGs.        
Although it is possible to formalize PDGS with hyperedges directly,
    we opt for a different approach here, in which PDGs have only regular edges,
and hyperedges are captured using a simple construction
that involves adding an extra node.

\vfull{\footnote{In the factor graph literature,
          especially with regard to loopy belief propagation
          \cite{wainwright2007graphical}, it is common to
          call a collection of marginals that are not
          necessarily all compatible with a distribution
          \emph{pseudomarginals}, making a PDG in some sense a
          collection of `conditional' pseudomarginals. This
          gives an alternate, more technically precise
          expansion of PDG as ``Pseudomarginal Dependency Graph''.}}

\begin{defn}\label{def:model}
A \emph{Probabilistic Dependency Graph}
is a tuple $\dg M = (\N,\Ed,\V,\mat p, \alpha, \beta)$, where 

\begin{description}%
	\item[$\N$] \notation{$:\Set$}%
		is a finite set of nodes, corresponding to variables;
	\item[$\Ed$] \notation{$\subseteq \N \times \N \times \mathit{Label}$}%
		is a set of labeled edges $\{ \ed LXY \}$, each with a source 
		$X$ and target $Y$ in $\N$;
	\item[$\V$] \notation{$\N \to \mathbf{Set}$}%
		associates each variable $N \in \N$ with a set $\V(N)$ of values that the variable $N$ can take;
  	\item[$\mat p$] \notation{$:\big(\!({A,B,\ell})\colon\!\Ed \big) \to \V(A) \to \Delta\V(B)$}%
	associates to each edge $\ed LXY \in \Ed$
	a distribution $\bp(x)$ on $Y$ for each $x \in \V(X)$; 

\item[$\alpha$] \notation{$:\Ed \to [0,1]$}
associates to each edge $\ed LXY$ a non-negative number $\alpha_L$ which,
roughly speaking, is the modeler's confidence in the functional
dependence of $Y$ on $X$ implicit in $L$; 
\item[$\beta$] \notation{$:\Ed \to \mathbb R^+$}
associates to each edge $L$ a positive real number $\beta_L$,
the modeler's 
subjective confidence in the reliability of
$\bp$. 
\end{description}
Note that we allow multiple edges in $\Ed$ with the same source and
target; thus $(\N,\Ed)$ is a multigraph.  We occasionally write a PDG
as $\dg M = (\Gr,\mat p, \alpha,\beta)$, where $\Gr = (\N,\Ed,\V)$, and
abuse terminology by referring to $\Gr$ as a multigraph.
We refer to 
${\dg N} = (\Gr, \mat p)$ as an \emph{unweighted} PDG,
and give it semantics as though it were the (weighted) PDG $(\Gr, \mat p, \mat 1, \mat 1)$, where
$\bf 1$ is the constant function (i.e., so that $\alpha_L = \beta_L = 1$ for all $L$). 
In this paper, with the exception of \cref{sec:expfam},  we implicitly take $\alpha = {\bf 1}$
and omit $\alpha$, writing $\dg M = (\Gr, \mat p, \beta)$.%
\footnote{The appendix gives results for arbitrary $\alpha$.} 
\end{defn}
If $\dg M$ is a PDG, we reserve the names 
$\N^{\dg M}, \Ed^{\dg M}, \ldots$,
for the components of $\dg M$, so that we may reference one without naming them
all explicitly. We write $\V(S)$ for the set of possible joint settings of a set
$S$ of variables, and write
$\V(\dg M)$ for all settings of the variables in $\N^{\dg M}$; we
refer  to these settings as ``worlds''.
While the definition above is sufficient to represent the class of all legal
PDGs, we often use two additional bits of syntax
to indicate common constraints:  
the special variable $\pdgunit$ such that $\V(\pdgunit)=\{\star\}$
from \Cref{ex:guns-and-floomps,ex:smoking}, and
double-headed arrows, $A \tto B$, which visually indicate 
that the corresponding cpd is degenerate, effectively representing a deterministic
function $f : \V(A) \to \V(B)$. 

\begin{constr}\label{constr:hyperedge-reducton}
We can now explain how we capture   the multi-tailed edges that 
were used in 
\Crefrange{ex:smoking}{ex:grok-union}. 
That notation can be viewed as shorthand for the graph that results by adding a new node at the junction representing the joint value of the nodes at the tails, with projections going back.  For instance,
the diagram displaying Grok's prior knowledge in \Cref{ex:grok-union}, on the left of \Cref{fig:grok-combine}
is really shorthand for the following PDG, where
where we insert a node labeled $C \times T$ at the junction:
\smallskip
	\begin{center}
	\ifprecompiledfigs
\raisebox{-0.5\height}{\includegraphics[scale=0.9]{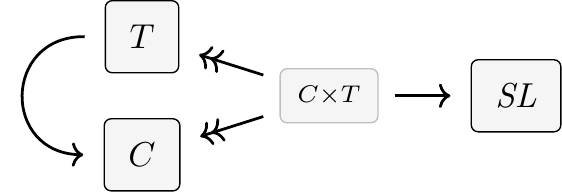}}
	\else
		\begin{tikzpicture}[paperfig]
			\node[dpadded] (SL) at (-1.0,0) {$\mathit{SL}$};
			
			\node[dpadded,light pad] (CT) at (-2.9, 0){$\scriptstyle C \times T$};
			\node[dpadded] (C) at (-4.8, -0.6) {$C$};
			\node[dpadded] (T) at (-4.8, 0.6) {$T$};

			\draw[arr, ->>] (CT) -- (C);
			\draw[arr, ->>] (CT) -- (T);
			\draw[arr] (CT) -- (SL);
			\draw[arr] (T) to [bend right=90, looseness=2] (C);
	\end{tikzpicture}
	\fi
	\end{center}
	\smallskip

As the notation suggests, $\V( C \times T) = \V(C) \times \V(T)$.
For any joint setting $(c,t) \in \V(C \times T)$ of both variables, the cpd for
the edge from $C \times T$ to $C$ gives probability 1 to $c$;
similarly, the cpd for the edge from $ C \times T$ to $T$ gives probability 1 to $t$.
\end{constr}

\section{Semantics}\label{sec:semantics}
Although the meaning of an individual cpd is clear, we have not yet given 
PDGs a ``global'' semantics. We discuss three related approaches to doing so.
The first is the simplest: we associate with a PDG the set of distributions that
are consistent with it. This set will be empty if the PDG is inconsistent.
The second approach associates a PDG with a scoring function, indicating the fit
of an arbitrary distribution $\mu$, and can be thought of as a \emph{weighted}
set of distributions \cite{HL12}. This approach allows us to distinguish
inconsistent PDGs, while the first approach does not. The third approach chooses
the distributions with the best score, typically associating with a PDG a unique
distribution.

\subsection{PDGs As Sets Of Distributions}\label{sec:set-of-distribution-semantics} 
We have been thinking of a PDG as a collection of constraints on distributions,
specified by matching cpds. From this perspective, it is natural to consider the
set of all distributions that are consistent with the constraints.

\begin{defn} \label{def:set-semantics} 
If $\dg M$ is a PDG (weighted or unweighted) with edges $\Ed$ and
cpds $\mat p$, 
let $\SD{\dg M}$ be the \emph{s}et of
\emph{d}istributions over the variables in $\dg M$ 
whose conditional marginals are exactly those given by $\mat p$.
That is, $\mu \in \SD{\dg M}$ iff, for all edges $L \in \Ed$ from $X$
to $Y$,  $x \in 
\V(X)$,  and $y \in \V(Y)$, we have that $\mu(Y \!=\! y \mid X\sheq x) = \bp(x)$.
\notation{Formally,		
    \[ \SD[\Big]{\dg M} = \!\left\{\mu \!\in\! \Delta \V_\none (\dg M) \middle|\!
        \begin{array}{l}
        \mu(B\!\! =\!\! b \mid A\!\!=\!\! a) \geq \bp(b \mid a) \\[0.1em]
        ~\text{$\forall (A, B,\ell) \!\in\! \Ed$, $a \!\in\!\V(A)$, $b \!\in\! \V(B)$} 
   		\end{array}\!\!\! \right\}\]
    }
$\dg M$ is \emph{inconsistent} if $\SD{\dg M} = \emptyset$, and \emph{consistent} otherwise.
\end{defn}
Note that $\SD{\dg M}$ is independent of the weights $\alpha$ and $\beta$.

\subsection{PDGs As Distribution Scoring Functions} \label{sec:scoring-semantics}   

We now generalize the previous semantics by viewing a PDG $\dg M$ as a
\emph{scoring function} that, given an arbitrary distribution $\mu$ on $\V(\dg
M)$, returns a real-valued score indicating how well $\mu$ fits $\dg M$.
Distributions with the lowest (best) scores are those that most closely match
the cpds in $\dg M$, and contain the fewest unspecified correlations.

We start with the first component of the score, which assigns higher scores to
distributions that require a larger perturbation in order to be consistent with
$\dg M$.  
We measure the magnitude of this perturbation with relative entropy. In
particular, for an edge $\ed LXY$ and $x \in \V(X)$, we measure
the relative entropy from $\bp(x)$ to $\mu(Y \!= \cdot\mid X=x)$, and take the
expectation over $\mu_X$ (that is, the marginal of $\mu$ on $X$). We then sum
over all the edges $L$ in the PDG, weighted by their reliability.

\begin{defn}\label{def:inc}
    For a PDG $\dg M$, the \emph{incompatibility} of a
    a joint distribution $\mu$ over $\V(\dg M)$, is given by
    \[
	\Inc_{\dg M }( \mu) := 
		\!\!\!\!\!\!\sum_{\ed L{X\,}{\,Y} \in \Ed^{\dg M}} \!\!\!\!\!\beta_L^{\dg M} \Ex_{x \sim \mu_{_X}}\!\!
\left[\kldiv[\Big]{ \mu(Y \mid X \sheq x) }{\bp^{\dg M}(x) } \right] , \]
	where $\kldiv{\mu}{\nu} = \sum_{w \in \text{Supp($\mu$)}} \mu(w) \log\frac{\mu(w)}{\nu(w)}$ is the 
	relative entropy from $\nu$ to $\mu$.
\vfull{
	The \emph{inconsistency of PDG $\dg M = \pdgvars[]$}, denoted $\Inc(\dg M)$, is the minimum possible incompatibility of $\dg M$ with any distribution $\mu$,  
	\[ \Inc(\dg M) = \inf_{ \mu \in \Delta [W_{\cal V}]} \Inc_{\dg M}(\mu) . \]
}
\end{defn}

$\SD{\dg M}$ and $\Inc_{\dg M}$ distinguish %
between distributions based on their compatibility with
$\dg M$, but even among distributions that match the
marginals, some more closely match the qualitative structure
of the graph than others.  
We think of each edge $\ed LXY$ as representing a
qualitative claim
(with confidence $\alpha_L$)
that the value of $Y$ can be computed from
$X$ alone.  
To formalize this, we require only the 
multigraph
$\Gr^{\dg M}$.

Given a multigraph $G$ and distribution $\mu$ on its variables,
contrast the amount of
information required to 
\begin{enumerate}[label=(\alph*)]
\item directly describe a joint outcome  \label{item:globalinfo}
$\mat w$ drawn from $\mu$, and 
\item separately specify, for each edge $\ed LXY$, the value
    $\mat w_Y$ (of $Y$ in world $\mat w$) 
	given the value $\mat w_X$, in expectation.
	\label{item:localinfo}
\end{enumerate}
If \ref{item:globalinfo} $=$ \ref{item:localinfo},
a specification of (b) has
 exactly the same length as a full desciption of the world. 
If \ref{item:localinfo} $>$ \ref{item:globalinfo}, then there are
correlations in $\mu$ that allow for a more compact representation
than $G$ provides. 
The larger the difference,  the more information is needed to determine
targets $Y$ beyond the conditional probabilities associated with the
edges $X \rightarrow Y$ leading to $Y$
(which according to $G$ should be sufficient to compute them), 
and the poorer the qualitative fit of $\mu$ to $G$.
Finally,
if \ref{item:globalinfo} $>$ \ref{item:localinfo}, then 
$\mu$ requires
additional information to specify, beyond
what is necessary to determine outcomes of the marginals selected by $G$.

\begin{defn}\label{def:info-deficiency}
For a multigraph $G = (\N, \Ed, \V)$ over a set $\N$ of variables,
define the \emph{$G$-information deficiency}
of distribution $\mu$, denoted $\IDef{G}(\mu)$,
by considering the difference between (a) and (b), 
where we measure the amount of information needed for a description
using entropy: 
\begin{equation}
	\IDef{G}(\mu) := \sum_{(X,Y) \in \Ed} \H_\mu(Y\mid X) - \H(\mu). 
	\label{eqn:alt-extra}
\end{equation}
(Recall that $H_\mu(Y\mid X)$, the
($\mu$-)\emph{conditional entropy of $Y$ given $X$}, is
defined as $- \sum_{x,y \in \V(X,Y)} \mu(x,y) \log \mu(y\mid x)$.)
For a PDG ${\dg M}$, we take $\IDef{\dg M} = \IDef{\Gr^{\dg M}}$.   
\end{defn}

We illustrate $\IDef{\dg M}$ with some simple examples.  
Suppose that $\dg M$ has two nodes, $X$ and $Y$.  If $\dg M$ has no edges, the
$\IDef{\dg M}(\mu) = - \H(\mu)$. There is no information required to specify, for
each edge in ${\dg M}$ from $X$ to $Y$, the value ${\mat w}_Y$ given ${\mat
w}_X$, since there are no edges. Since we view smaller numbers as representing a
better fit, $\IDef{\dg M}$ in this case will prefer the distribution that
maximizes entropy. If $\dg M$ has one edge from $X$ to $Y$, then since $\H(\mu) =
\H_{\mu}(Y \mid X) + \H_\mu(X)$
by the well known \emph{entropy chain rule} \cite{mackay2003information},
$\IDef{\dg   M}(\mu) = -\H_{\mu}(X)$. Intuitively,
while knowing the conditional probability $\mu(Y \mid X)$ is helpful, to
completely specify $\mu$ we also need $\mu(X)$. Thus, in this case, $\IDef{\dg
M}$ prefers distributions that maximize the entropy of the marginal
on $X$. 
If $\dg M$ has sufficiently many parallel edges
 from $X$ to $Y$
 and $\H_{\mu}(Y \mid X) > 0$ 
(so that $Y$ is not totally determined by $X$)
then we have $\IDef{\dg M}(\mu) > 0$, because the redundant edges add no
information, but there is still a cost to specifying them.
In this case, $\IDef{\dg M}$ prefers distributions that make $Y$ a
deterministic function of $X$ will maximizing the entropy of the
marginal on $X$.
Finally, if ${\dg M}$ has an edge from $X$ to $Y$ and another from $Y$
to $X$, then a distribution $\mu$ minimizes $\IDef{\dg M}$ when 
$X$ and $Y$  vary together (so that $\H_\mu(Y \mid X) = \H_\mu(X \mid Y) = 0$)
while maximizing $\H(\mu)$, for example, by taking $\mu(0,0)
= \mu(1,1) = 1/2$.

$\Inc_{\dg M}(\mu)$ and $\IDef{\dg M}(\mu)$ give us two measures
of compatibility between ${\dg M}$ and a distribution $\mu$.
We take the score of interest to be their sum, with the tradeoff
controlled by a parameter $\gamma \ge 0$:

\begin{equation}
  	  \bbr{\dg M}_\gamma(\mu)
	 := \Inc_{\dg M}(\mu) + \gamma \IDef{\dg M}(\mu)  \label{eqn:full-score}
\end{equation}

The following just makes precise that the scoring semantics generalizes the first semantics.

\begthm{prop}{prop:sd-is-zeroset}
	$\SD{\dg M} \!= \{ \mu : \bbr{\dg M}_0(\mu) \!=\! 0\}$ for 
	all $\dg M$.
\end{prop}
          
While we focus on this particular scoring function in the paper, 
in part because
it has deep connections to the free energy of a factor graph \cite{KF09},
other scoring functions may well end up being of interest.

\subsection{PDGs As Unique Distributions}\label{sec:uniq-dist-semantics}

Finally, we provide an interpretation of a PDG as a probability distribution.
Before we provide this semantics, we stress that this distribution does
\emph{not} capture all of the important information in the PDG---for example, a
PDG can represent inconsistent knowledge states.  Still, by giving a
distribution, we enable comparisons with other graphical models%
, and show that PDGs are
a surprisingly flexible tool for specifying distributions.  
The idea is to select the distributions with the best score. 
We thus define 
\begin{equation}
	\bbr{\dg M}_\gamma^* = \argmin_{\mu \in \Delta\V(\dg M)} \bbr{\dg M}_\gamma(\mu).
\end{equation}   

In general, $\bbr{\dg M}_\gamma^*$ does not give a unique distribution.  But if
$\gamma$ is sufficiently small, then it does:

\begthm{prop}{prop:sem3}
	If $\dg M$ is a PDG and $0 < \gamma \leq \min_L \beta_L^{\dg M}$, then
	$\bbr{\dg M}_\gamma^*$ is a singleton. 
\end{prop}

In this paper, we are interested in the case where $\gamma$ is small;
this amounts to emphasizing the accuracy of the probability
distribution as a description of probabilistic information,
rather than the graphical structure of the PDG.  
This motivates us to consider
what happens as $\gamma$ goes to 0.  If $S_\gamma$ is a set of
probability distributions for all $\gamma \in [0,1]$, we define $\lim_{\gamma
\rightarrow 0} S_\gamma$ to consist of all distributions $\mu$ such that there
is a sequence $(\gamma_i, \mu_i)_{i \in \mathbb N}$ with $\gamma_i \to 0$ and
$\mu_i \to \mu$ such that $\mu_i \in S_{\gamma_i}$ for all $i$. 
It can be further shown that 

\begthm{prop}{prop:limit-uniq}
    For all $\dg M$, $\lim_{\gamma\to0}\bbr{\dg M}_\gamma^*$ is a singleton.
\end{prop}
Let $\bbr{{\dg M}}^*$ be the unique element of $\smash{\lim\limits_{\gamma
	\rightarrow 0}} \bbr{{\dg M}}_\gamma^*$. 
The semantics has an important property: 

\begthm{prop}{prop:consist}
	$\bbr{\dg M}^* \in \bbr{\dg M}_0^*$, so if $\dg M$ is consistent,
	then $\bbr{\dg M}^* \in \SD{\dg  M}$.
\end{prop}

\section{Relationships to Other Graphical Models}
\label{sec:other-graphical-models} 
We start by relating
PDGs to two of the most popular graphical models: BNs and factor
graphs. PDGs are strictly more general than BNs, and can emulate factor graphs
for a particular value of $\gamma$. 
\subsection{Bayesian Networks} 
\label{sec:bn-convert}

\Cref{constr:hyperedge-reducton} can be generalized to convert arbitrary Bayesian Networks into PDGs.
Given a BN $\mathcal B$ and a positive confidence $\beta_X$ for
the cpd of each variable $X$ of $\cal B$,
let $\PDGof{\mathcal B, \beta}$
be the PDG comprising the cpds of $\cal B$
in this way; we defer the straightforward formal details to the appendix.

\begthm{theorem}{thm:bns-are-pdgs}
 	  If $\cal B$ is a Bayesian network
          and $\Pr_{\cal B}$ is the distribution it specifies, then
        for all $\gamma > 0$ and all vectors $\beta$ such
        that $\beta_L > 0$ for all edges $L$,
        $\bbr{\PDGof{\mathcal B, \beta}}_\gamma^* = \{ \Pr_{\cal B}\}$, 
and thus $\bbr{\PDGof{\mathcal B, \beta}}^* = \Pr_{\cal B}$.    
\end{theorem}
\Cref{thm:bns-are-pdgs} is quite robust to parameter choices: it holds for every
weight vector $\beta$ and all $\gamma > 0$. However, it does lean heavily on
our assumption that $\alpha = \mathbf 1$, making it our only result
that does not 
have a natural analog for general $\alpha$.

\subsection{Factor Graphs} 
\label{sec:factor-graphs}
Factor graphs 
\cite{kschischang2001sumproduct},
like PDGs, generalize BNs.
In this section, we consider the relationship between factor graphs (FGs) and PDGs.
\begin{defn}
 A \emph{factor graph} $\Phi$ is a set of random variables
        $\mathcal X = \{X_i\}$ and \emph{factors}
       $\{\phi_J\colon \V(X_J) \to \mathbb R_{\geq0}\}_{J \in
\mathcal J }$,
where $X_J \subseteq \mathcal X$.  
More precisely, each factor $\phi_J$ is associated with a subset
$X_J\subseteq \mathcal{X}$ of variables, and maps
joint settings of $X_J$ to non-negative real numbers.
$\Phi$ specifies a distribution
\[ {\Pr}_{\Phi}(\vec x) = \frac{1}{Z_{\Phi}}
 	\prod_{J \in \cal J} \phi_J(\vec x_J), \]
where $\vec{x}$ is a joint setting of all of the variables,
 $\vec{x}_J$ is the restriction of $\vec{x}$ to only the
 variables $X_J$, and $Z_{\Phi}$ is the constant required to
 normalize the distribution.  
\end{defn}

The cpds of a PDG naturally constitute a collection of factors,
so it natural to wonder how the semantics of a PDG compares to 
simply treating the cpds as factors in a factor graph. To answer this,
we start by making the translation precise.
\begin{defn}[unweighted PDG to factor graph]\label{def:PDG2fg}
If $\dg N = (\Gr, \mat p)$ is an unweighted PDG, define   
the associated FG $\FGof{\dg N}$ on the 
variables $(\N, \V)$ by
taking $\mathcal J$ to be the set of edges, 
and for an edge $L$ from $Z$ to $Y$, taking $X_{L} = \{Z,Y\}$, and $\phi_L(z,y)$ to be
$\bp^{\dg M}(y \mid z)$ (i.e., $(\bp^{\dg M}(z))(y)$).
\end{defn}

It turns out we can also do the reverse. 
Using essentially the same idea as in \cref{constr:hyperedge-reducton},
we can encode a factor graph as an assertion about the unconditional
probability distribution over the variables associated to each
factor.  

\begin{defn}[factor graph to unweighted PDG] \label{def:fg2PDG}
For a FG $\Phi$, let $\UPDGof{\Phi}$ be
the unweighted PDG consisting of
\begin{itemize}
	\item the variables in $\Phi$ together
   with $\pdgunit$ and a variable $X_{\!J} := \prod_{j \in J} X_j$ for every factor $J \in \mathcal J$%
   , and
   \item edges ${\pdgunit} \!\!\to\! X_{\!J}$ for each $J$ and $X_{\!J} \!\!\tto\! X_j$ for each $X_j \in \mat X_J$,
\end{itemize}
where the edges $ X_{\!J} \!\tto\! X_j$ are associated with the appropriate projections, and each ${\pdgunit} \!\to\! X_{\!J}$ is associated with the unconditional joint distribution on $X_J$ obtained by normalizing $\phi_J$. 
The process is illustrated in \cref{fig:fg2PDG}.
\end{defn}

\begin{figure*}[htb]
	\centering
	\hfill
	\ifprecompiledfigs
\raisebox{-0.5\height}{\includegraphics{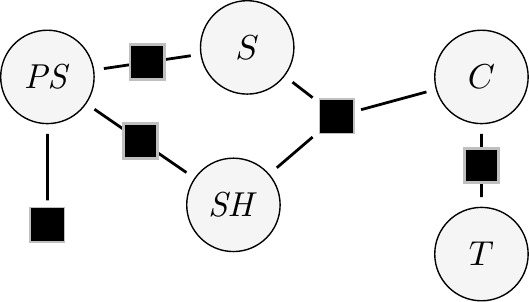}}
	\else
	\begin{tikzpicture}[center base, xscale=1.4,
		fgnode/.append style={minimum width=2.7em, inner sep=0.3em}]
		\node[factor] (prior) at (1.65,-1) {};
		\node[factor] (center) at (3.75, 0.1){};
		
		\node[fgnode] (PS) at (1.65,0.5) {$\mathit{PS}$};
		\node[fgnode] (S) at (3.1, 0.8) {$S$};
		\node[fgnode] (SH) at (3.0, -0.8) {$\mathit{SH}$};
		\node[fgnode] (C) at (4.8,0.5) {$C$};
		
		\draw[thick] (prior) -- (PS);
		\draw[thick] (PS) --node[factor](pss){} (S);
		\draw[thick] (PS) --node[factor](pssh){} (SH);
		\draw[thick] (S) -- (center) (center) -- (SH) (C) -- (center);

		\node[fgnode] (T) at (4.8, -1.3) {$T$};
		\draw[thick] (T) -- node[factor]{}  (C);	
	\end{tikzpicture}
	\fi
	\hfill\vrule\hfill
	\ifprecompiledfigs
\raisebox{-0.5\height}{\includegraphics{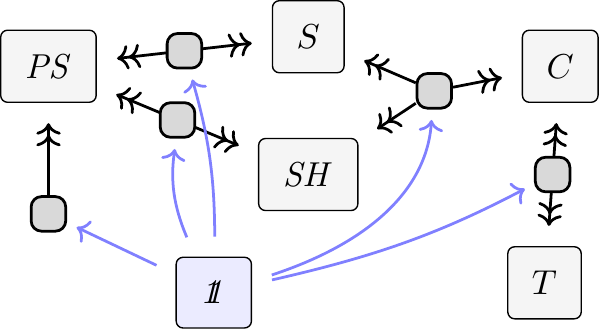}}
	\else
	\begin{tikzpicture}[center base, xscale=1.6,
        newnode/.style={rectangle, inner sep=5pt, fill=gray!30, rounded corners=3, thick,draw}]
		\node[newnode] (prior) at (1.65,-1) {};
		\node[newnode] (center) at (4.1, 0.25){};
		
		\node[dpadded] (PS) at (1.65,0.5) {$\mathit{PS}$};
		\node[dpadded] (S) at (3.3, 0.8) {$S$};
		\node[dpadded] (SH) at (3.3, -0.6) {$\mathit{SH}$};
		\node[dpadded] (C) at (4.9,0.5) {$C$};
		
		\draw[arr, ->>, shorten <=0pt] (prior) -- (PS);
		\draw[arr, <<->>] (PS) --node[newnode](pss){} (S);
		\draw[arr, <<->>] (PS) --node[newnode](pssh){} (SH);
		\draw[arr, <<-, shorten >=0pt] (S) -- (center); 
		\draw[arr, <<-, shorten >=0pt] (SH)-- (center); 
		\draw[arr, <<-, shorten >=0pt] (C) -- (center);
		
		\node[dpadded, fill=blue] (1) at (2.7,-1.8) {$\pdgunit$};
		
		\draw[blue!50, arr] (1) -- (prior);
		\draw[blue!50, arr] (1) to[bend right=30] (center);
		\draw[blue!50, arr] (1) to[bend right = 5] (pss);
		\draw[blue!50, arr] (1) to[bend left = 10] (pssh);

		\node[dpadded] (T) at (4.8, -1.7) {$T$};
		\draw[arr, <<->>] (T) -- node[newnode](tc){}  (C);	

		\draw[blue!50, arr] (1) to[bend right = 10] (tc);
	\end{tikzpicture}
	\fi
	\hfill~
	\caption{
Conversion of the PDG in \cref{ex:smoking} to a factor graph
according to \cref{def:PDG2fg} (left), and from that factor graph back
to a PDG by \cref{def:fg2PDG} (right). 
In the latter, for each $J$ we introduce a new variable $X_J$ (displayed as a
smaller darker rectangle), whose values are joint settings of the
variables connected it, and also an edge $1 \to X_J$ 
(shown in blue),
to which we associate the unconditional 
distribution given by normalizing $\phi_J$.
} 
	\label{fig:fg2PDG}
\end{figure*}

PDGs are directed graphs, while factors graphs are undirected. The
map from PDGs to factor graphs thus loses some important structure.
As shown in \Cref{fig:fg2PDG},
this mapping can change the graphical structure significantly.
Nevertheless,

\begthm{theorem}{thm:fg-is-pdg}
$\Pr_{\Phi} = \bbr{\UPDGof{\Phi}}_{1}^*\;$ for all factor graphs
$\Phi$.\footnote{Recall that we identify the unweighted PDG $(\Gr,\mat
p)$ with the weighted PDG $(\Gr,\mat p,  \mat 1, \mat 1)$.}  
\end{theorem}
\begthm{theorem}{thm:pdg-is-fg}
$\bbr{\dg N}_{1}^* = \Pr_{\FGof{\dg N}}\;$ for all unweighted
	PDGs $\dg N$.  
\end{theorem}
The correspondence hinges on the fact that we take $\gamma=1$, so that $\Inc$ and
$\IDef{}$ are weighted equally.
Because the user of a PDG gets to choose $\gamma$, the fact that the 
translation from factor graphs to PDGs preserves semantics only for $\gamma=1$
poses no problem.
Conversely, the fact that the reverse correspondence requires
$\gamma=1$ suggests 
that factor graphs are less flexible than PDGs.

What about weighted PDGs $(\Gr, \mat p, \beta)$ where $\beta \ne {\bf 1}$?
There is also a standard notion of weighted factor graph,
but as long as we stick with our convention of taking  $\alpha = {\bf 1}$, 
we cannot relate them to weighted PDGs.  
As we are about to see,
once we drop this convention, we can do much more.

\subsection{Factored Exponential Families}\label{sec:expfam}

A \emph{weighted factor graph (WFG)} $\Psi$ is a pair
$(\Phi,\theta)$ consisting of a factor graph $\Phi$ 
together with a vector of non-negative weights
$\{ \theta_J \}_{J \in \mathcal J}$.
$\Psi$ specifies a canonical scoring function 
\begin{equation}
\GFE_{\Psi}(\mu)
	 := \!\Ex_{\vec x\sim\mu}\left[  \sum_{J \in
           \cal J} \theta_J \log\frac1{\phi_J(\vec
               x_J)}\right] - \H(\mu)  , 
			   \label{eqn:free-energy}
\end{equation}
called the \emph{variational
Gibbs free energy} \cite{mezard2009information}. 
$\GFE_{\Psi}$ is uniquely minimized by the distribution
${\Pr}_{\Psi}(\vec x) = \frac{1}{Z_{\Psi}}
 	\prod_{J \in \cal J} \phi_J(\vec x_J)^{\theta_J}$, 
which matches the unweighted case when every $\theta_J = 1$.
The mapping $\theta \mapsto \Pr_{(\Phi,\theta)}$ is known as 
$\Phi$'s \emph{exponential family} and is a central tool in the analysis  
and development of many algorithms for graphical models \cite{wainwright2008graphical}.

PDGs can in fact capture the full exponential family of a factor graph, but only
by allowing values of $\alpha$ other than ${\bf 1}$. In this case, the
only definition  
that requires alteration is $\IDef{}$, which now depends on the \emph{weighted multigraph}
$(\Gr^{\dg M}, \alpha^{\dg M})$, and is given by
\begin{equation}
	\IDef{\dg M}(\mu) := \sum_{\ed LXY \in \Ed} \alpha_L \H_\mu(Y\mid X) - \H(\mu). 
	\label{eqn:alt-extra2}
\end{equation}
Thus, the conditional entropy $\H_\mu(Y\mid X)$ associated with the
edge $\ed LXY$ is multiplied by the weight $\alpha_L$ of that edge.

One key benefit of using $\alpha$ is that we can
capture arbitrary WFGs, not just ones with a constant weight
vector.    All we have to do is to ensure that in our translation from
factor graphs to PDGs, the ratio $\alpha_L/\beta_L$ is a
constant.  (Of course, if we allow arbitrary weights, we cannot hope
to do this if $\alpha_L = 1$ for all edges $L$.)  
We therefore define a family of translations, parameterized by the
ratio of $\alpha_L$ to $\beta_L$.
\begin{defn}[WFG to PDG]\label{def:wfg2pdg}
Given a WFG
$\Psi=(\Phi, \theta)$,
and postive number $k$, 
we define the corresponding PDG $\PDGof{\Psi,k} = (\UPDGof{\Phi},\alpha_{\theta}, \beta_{\theta})$ 
by taking $\beta_J = k \theta_J$ and $\alpha_J = \theta_J$ for the edge
 $\pdgunit  \rightarrow X_J$, and
taking $\beta_L = k$ and $\alpha_L = 1$ for the projections $X_J \!\tto\! X_j$.
\end{defn}

We now extend Definitions~\ref{def:PDG2fg} and \ref{def:fg2PDG} to
(weighted) PDGs and WFGs.  
In translating a PDG to a WFG, 
there will necessarily be some loss of information: PDGs have two sets, while WFGs have 
only have one. Here we throw out $\alpha$ and keep $\beta$, 
though in its role here as a left inverse of \cref{def:wfg2pdg},
either choice would suffice.

\begin{defn}[PDG to WFG]
Given a (weighted) PDG $\dg M =
(\dg N, \beta)$, we take its corresponding WFG to be $\WFGof{\dg M} :=
(\FGof{\dg N}, \beta)$; that is, $\theta_L := \beta_L$ for all edges $L$.
\end{defn}
We now show that we can capture the entire exponential family of a factor graph,
and even its associated free energy, 
but only for $\gamma$ equal to the constant $k$ used in
the translation.

\begin{theorem}\label{thm:wfg-is-pdg}
For all WFGs $\Psi = (\Phi,\theta)$ and all $\gamma > 0$,
we have that
$\GFE_\Psi
= \nicefrac1{\gamma} \bbr{{\dg M}_{\Psi,\gamma}}_{\gamma} 
+ C$   
for some constant $C$, so
$\Pr_{\Psi}$ is the unique element of
$\bbr{{\dg M}_{\Psi,\gamma}}_{\gamma}^*$.
\end{theorem}

In particular, for $k\!=\!1$, so that $\theta$ is used for both the functions
$\alpha$ and $\beta$ of the resulting PDG,
\cref{thm:wfg-is-pdg} strictly generalizes \cref{thm:fg-is-pdg}.
\begin{coro}
	For all weighted factor graphs $(\Phi, \theta)$,
	we have that
	$\Pr_{(\Phi,\theta)} = \bbr{(\UPDGof{\Phi}, \theta,\theta)}_1^*$
\end{coro}

Conversely, as long as the ratio of $\alpha_L$ to $\beta_L$ is constant, the
reverse translation also preserves semantics.
\begthm{theorem}{thm:pdg-is-wfg}
For all unweighted PDGs $\dg{N}$ and non-negative vectors $\mat v$
over $\Ed^{\dg N}$, and all $\gamma > 0$, we have that 
$\bbr{(\dg N, \mat v, \gamma \mat v)}_{\gamma}
= \gamma\,\GFE_{(\Phi_{\dg N}, \mat v)} $; consequently,
$\bbr{(\dg N,  \mat v,  \gamma\mat v)}_{\gamma}^*
		= \{\Pr_{(\Phi_{\dg N}, \mat v)} \}$. 
\end{theorem}

The key step in proving \Cref{thm:wfg-is-pdg,thm:pdg-is-wfg}
(and in the proofs of a number of other results) involves 
rewriting  
$\bbr{\dg M}_\gamma$ as follows: 
\begin{prop}[restate=prop:nice-score,label=prop:nice-score]%
 Letting $x^{\mat w}$ and $y^{\mat w}$ denote the values of
  $X$ and $Y$, respectively, in $\mat w \in \V(\dg M)$, 
we have 
\begin{equation}\label{eq:semantics-breakdown}
\begin{split}
\bbr{\dg M}(\mu) =  \Ex_{\mat w \sim \mu}\! \Bigg\{
 \sum_{ X \xrightarrow{\!\!L} Y  }
\bigg[\,
    \color{gray}\overbrace{\color{black}
      \!\beta_L \log \frac{1}{\bp(y^{\mat w} |x^{\mat w})}
	}^{\color{gray}\smash{\mathclap{\text{log likelihood / cross entropy}}}} + \qquad\\[-0.7em]
    \color{gray}\underbrace{\color{black} 
({\alpha_L}\gamma - \beta_L ) \log \frac{1}{\mu(y^{\mat w} |x^{\mat w})} 
	}_{\color{gray}\smash{\mathclap{\text{local regularization (if $\beta_L > 
	\alpha_L
	\gamma$)}}}}\bigg] - \underbrace{\color{black}
\gamma \log \frac{1}{\mu(\mat w)}
	}_{\color{gray}\smash{\mathclap{\text{global
        regularization}}}}\color{black} \Bigg\} .
\end{split}
\end{equation}
\end{prop}
For a fixed $\gamma$, the first and last terms
of \eqref{eq:semantics-breakdown} are equal to a scaled
version of the free energy, $\gamma\GFE_\Phi$, 
if we set $\phi_J := \bp$ and $\theta_J := \nicefrac{\beta_L}{\gamma}$.  
If, in addition, $\beta_L = {\alpha_L}\gamma$ for all
edges $L$, then
the local regularization term disappears, giving us
the desired correspondence.

\Cref{eq:semantics-breakdown} also makes it clear that 
taking $\beta_L = {\alpha_L} \gamma$ for all edges $L$ is
essentially necessary to get \Cref{thm:pdg-is-fg,thm:fg-is-pdg}.
Of course, fixed $\gamma$ precludes taking the limit as $\gamma$ goes
to 0, so 
\cref{prop:consist} does apply. This is reflected in 
some strange
behavior in factor graphs trying to capture the same phenomena as
PDGs, as the following example shows.

\begin{example}\label{ex:overdet}
Consider the PDG $\dg M$ containing just $X$ and $1$, and two edges
$p, q: 1 \to X$.
(Recall that such a PDG can arise if we get different information about the
probability of $X$ from two different 
sources; this is a situation we
certainly want to be able to capture!)
Consider the simplest situation, where $p$ and $q$ are both associated
with the same distribution on $X$%
; further suppose that the agent is certain about the distribution, so
$\beta_p = \beta_q = 1$.
For definiteness, suppose that
$\V(X) = \{x_1,x_2\}$, and
that the distribution associated with both edges is $\mu_{.7}$, which ascribes
probability $.7$ to $x_1$. Then, as we would hope  $\bbr{\dg M}^* =
\{\mu_{.7}\}$; after all, both sources agree on the information.
However, it can be shown that 
$\Pr_{\WFGof{\dg M}} = \mu_{.85}$, so  $\bbr{\dg M}_1^* = \{\mu_{.85}\}$.
\end{example}

Although both $\theta$ and $\beta$ are measures of confidence, 
the way that the Gibbs free energy varies with $\theta$ 
is quite different from the way that the score of a PDG
varies with $\beta$. 
The scoring function that we use for PDGs can be viewed as
extending ${\GFE}_{\Phi,\theta}$ by including
the local regularization term.
As $\gamma$ approaches zero,
the importance of the global regularization terms decreases relative
to that of the local regularization term, so the PDG scoring function
becomes quite different from Gibbs free energy.

\section{Discussion}
We have introduced PDGs, a powerful tool for representing
probabilistic information. 
They have a number of advantages over other
probablisitic graphical models. 
\begin{itemize}
\item They allow us to capture inconsistency, including conflicting information
from multiple sources with varying degrees of reliability.
\item 
	They are much more modular 
	than other representations; for example, we can combine information from two sources by simply taking the union of two
	PDGs, and it is easy to add new information (edges)
	and features (nodes) without affecting previously-received information.
\item They allow for a clean separation between quantitiatve information (the
	cpds and weights $\beta$) and more qualitative information contained by
	the graph structure (and the weights $\alpha$); this is captured by the
	terms $\Inc$ and $\IDef{}$ in our scoring function.
\item PDGs have (several) natural semantics; one of them allows us to
pick out a unique distribution.  Using this distrbution, PDGs
	can capture BNs and factor graphs.
In the latter case, a simple parameter shift in the corresponding PDG eliminates
arguably problematic behavior of a factor graph.
\end{itemize}

We have only scratched the surface of what can be done with PDGs here.
Two major issues that need to be tackled are inference and dynamics.
How should we query a PDG for probabilistic information? How should 
we modify a PDG in light of new information or to make it more consistent?
These issues turn out to be closely related.
Due to space limitations, 
we just briefly give some intuitions and examples here.

Suppose that we want to compute the probability of $Y$ given $X$ in a PDG $\dg
M$. For a cpd $p(Y|X)$, let $\dg M^{+p}$ be the PDG obtained by associating $p$
with a new edge in $\dg M$ from $X$ to $Y$, with $\alpha_p \!=\! 0$. We judge
the quality of a candidate answer $p$ by the best possible score that $\dg
M^{+p}$ gives to any distribution (which we call the \emph{degree of
inconsistency} of $\dg M^{+p}$). It can be shown that the deegree of
inconsistency is minimized by $\bbr{\dg M}^*(Y\!\mid\! X)$. Since the degree of
inconsistency of $\dg M^{+p}$ is smooth and strongly convex as a function of
$p$, we can compute its optimum values by standard gradient methods. This
approach is inefficient as written (since it involves computing the full joint
distribution $\bbr{{\dg M}^{+p}}^*$), but we believe that standard approximation
techniques will allow us to draw inferences efficiently.

	To take another example,
	conditioning can be understood in terms of resolving inconsistencies
	in a PDG.  To condition on an observation $Y\!=\!y$, given a situation
	described by a PDG $\dg M$, we can add an edge 
	from $\pdgunit$ to $Y$ in $\dg M$, annoted with the cpd that gives
	probability 1 to $y$, to get the (possibly inconsistent) PDG
	$\vphantom{\raisebox{0.05em}{$\big|$}}$
${\dg M}^{+(\!Y\!=y)}$.   The distribution $\bbr{{\dg M}^{+(\!Y\!=y)}}^*$ turns out
to be the result of conditioning $\bbr{\dg M}^*$ on $Y\!=\!y$.
This account of conditioning generalizes
without modification to give Jeffrey's Rule \cite{Jeffrey68}, a more
general approach to belief updating. 

Issues of updating and inconsistency also arise in
variational inference. A  
variational autoencoder \cite{kingma2013autoencoding}, for instance, 
is essentially three cpds: a prior $p(Z)$, a decoder $p(X \!\mid\! Z)$, and 
an encoder $q(Z \!\mid\! X)$. Because two cpds target $Z$ (and the cpds are 
inconsistent until fully trained), this situation
can be represented by PDGs but not
by other graphical models.
We hope to report further on the deep connection between
inference, updating, 
and the resolution of inconsistency in PDGs
in future work.

\clearpage
{
    \bibliography{allrefs,z,joe}        
}

\section*{Ethics Statement}
Because PDGs are a recent theoretical development, there is a lot of
guesswork in evaluating the impact. Here are two views of opposite polarity.

\subsection{Positive Impacts}
One can imagine many applications of enabling simple and coherent
aggregation of (possibly inconsistent) information. In particular we
can imagine using PDGs to build and interpret a communal and global
database of statistical models, in a way that may not only enable more
accurate predictions, but also highlights conflicts between
information.

This could have many benefits. Suppose, for instance, that two researchers train
models, but use datasets with different racial makeups. Rather than trying to
get an uninterpretable model to ``get it right'' the first time, we could simply
highlight any
such clashes and flag them for review.

Rather than trying to ensure fairness by design, which is both tricky and
costly, we envision an alternative: simply aggregate (conflicting) statistically
optimal results, and allow existing social structure to resolve conflicts,
rather than sending researchers to fiddle with loss functions until they look fair.

\subsection{Negative Impacts}

We can also imagine less rosy outcomes. To the extent that PDGs can
model and reason with inconsistency, if we adopt the attitude that a PDG need
not wait until it is consistent to be used, it is not hard to imagine a world
where a PDG gives biased and
poorly-thought out conclusions. It is clear that PDGs need a great deal more
vetting before they can be used for such important purposes as
aggregating the world's statistical knowledge.

PDGs are powerful statistical models, but are by necessity semantically more
complicated than many existing methods. This will likely restrict their
accessibility. To mitigate this, we commit to making sure our work is widely
accessible to researchers of different
backgrounds.
\fi
\ifappendix
\appendix
\clearpage
\onecolumn
\section{Proofs} \label{sec:proofs}
		For brevity, we use the standard notation and write $\mu(x, y)$
	instead of $\mu(X \!=\! x, Y \!=\! y)$, $\mu(x \mid y)$ instead of
	$\mu(X \!=\! x\mid Y \!=\! y)$, and so forth.

\subsection{Properties of Scoring Semantics}

In this section, we prove the properties of scoring functions that we
mentioned in the main text,
Propositions~\ref{prop:sd-is-zeroset}, \ref{prop:sem3}, and
\ref{prop:consist}.  We repeat the statements for the reader's convenience.

\restate{prop:sd-is-zeroset}{
$\SD{\dg M} \!= \{ \mu : \bbr{\dg M}_0(\mu) \!=\! 0\}$ for all $\dg M$.
}
\begin{proof}
	 By taking $\gamma = 0$, the score is just $\Inc$. By
			 definition, a distribution $\mu \in \SD{\dg M}$ satisfies
	  all the
			 constraints, so $\mu(Y = \cdot \mid X=x) =
			 \bp(x)$ for all edges $X \rightarrow Y \in \Ed^{\dg
			   M}$ and $x$ with 
			 $\mu(X=x)>0$. By Gibbs inequality
			 \cite{mackay2003information}, 
			 $\kldiv{\mu(Y|x)}{\bp(x)} = 0$. Since this is true
			 for all edges, we must have $\Inc_{\dg M}( \mu) =
			 0$. Conversely, if $\mu \notin \SD{\dg M}$, then it
			 fails to marginalize to the cpd $\bp$ on some edge
							  $L$, and so again by Gibbs inequality,
			 $\kldiv{\mu(Y|x)}{\bp(x)} > 0$. As relative entropy
			 is non-negative, the sum of these terms over all
			 edges must be positive as well, and so $\Inc_{\dg M}(
			 \mu) \neq 0$. %
\end{proof}

Before proving the remaining results, we prove a lemma that will be useful
in other contexts as well. 

\begin{lemma}
	\label{thm:inc-convex}
	$\Inc_{\dg M}( \mu)$ is a convex function of $\mu$.
\end{lemma}
\begin{proof}
    It is well known that $\thickD$ is convex \cite[Theorem
            2.7.2]{coverThomas}, in the sense that  
	\[ \kldiv{\lambda q_1 + (1-\lambda) q_2 }{ \lambda p_1
			  + (1-\lambda) p_2} \leq \lambda \kldiv {q_1}{ p_1} +
							(1-\lambda) \kldiv{q_2}{p_2}. \] 
Given an edge $L \in \Ed$ from $A$ to $B$ and $a \in \mathcal V(A)$,
and   
setting $q_1 = q_2 = \bp(a)$, we get that
	\[ \thickD(\bp(a) \ ||\ \lambda p_1 + (1-\lambda) p_2)
			\leq \lambda \thickD (\bp(a) \ ||\ p_1) + (1-\lambda)
							\thickD(\bp(a)\ ||\ p_2). \] 
	Since this is true for every $a$ and edge, we can take
		   a weighted sum of these inequalities for each $a$
		   weighted by $p(A=a)$; thus, 
	\begin{align*}
		\Ex_{a\sim p_A} \kldiv{\bp(a)}{\lambda p_1 +
			(1-\lambda) p_2} &\leq 
			 \Ex_{a\sim p_A}\lambda \kldiv {\bp(a)}{p_1} +
											(1-\lambda)
                         			 \kldiv{\bp(a)}{p_2}. \\
                        \intertext{Taking a sum over all edges, we get
                        that}
					\sum_{(A, B) \in \Ed}\mskip-10mu\Ex_{a\sim p_A} \kldiv{\bp(a) }{\lambda p_1 + (1-\lambda) p_2} 
			&\leq \sum_{(A, B) \in
							  \Ed}\mskip-10mu\Ex_{a\sim p_A}\lambda
							\kldiv{\bp(a)}{p_1} + (1-\lambda)
							\kldiv{\bp(a)}{p_2}. \\
    \shortintertext{It follows that} 
		\Inc_{\dg M}( \lambda p_1) + (1-\lambda)p_2)
					&\leq \lambda \Inc_{\dg M}(p_1) + (1-\lambda)
					\Inc_{\dg M}(p_2). 
	\end{align*}
	Therefore, $\Inc_{\dg M}( \mu)$ is a convex function of $\mu$.
\end{proof}

The next proposition gives us a useful representation of $\bbr{M}_\gamma$.
\restate{prop:nice-score}{
Letting $x^{\mat w}$ and $y^{\mat w}$ denote the values of
 $X$ and $Y$, respectively, in $\mat w \in \V(\dg M)$, 
we have 
\begin{equation*}
\begin{split}
\bbr{\dg M}(\mu) =  \Ex_{\mat w \sim \mu}\! \Bigg\{
\sum_{ X \xrightarrow{\!\!L} Y  }
\bigg[\,
	 \!\beta_L \log \frac{1}{\bp(y^{\mat w} |x^{\mat w})}
   +
({\alpha_L}\gamma - \beta_L ) \log \frac{1}{\mu(y^{\mat w} |x^{\mat w})} 
 \bigg] - 
	\gamma \log \frac{1}{\mu(\mat w)}
   \Bigg\} .
\end{split}
\end{equation*}
}
\begin{proof}
We use the more general formulation of $\IDef{}$ given
in \Cref{sec:expfam}, in which each edge $L$'s conditional  
information is weighted by $\alpha_L$.
  \begin{align*}
	\bbr{\dg M}_\gamma(\mu) &:= \Inc_{\dg M}( \mu) + \gamma \IDef{\dg M}(\mu) \\
		&= \left[\sum\alle \beta_L \Ex_{x\sim \mu_X}\kldiv[\Big]{ \mu(Y | X \sheq x) }{\bp(x) } \right]  + \gamma \left[\sum\alle \alpha_L \H_\mu(Y\mid X) ~-\H(\mu)\right]\\
		&= \sum\alle 
			\Ex_{x \sim \mu_{\!_X}}  \left[ \beta_L\; \kldiv[\Big]{ \mu(Y \mid x) }{\bp(Y \mid x) } + \gamma \; \alpha_L \H(Y \mid X\sheq x) \right]  - \gamma \H(\mu) \\ 
		&= \sum\alle 
			\Ex_{x \sim \mu_{\!_X}}  \left[ \beta_L\; \left(\sum_{y \in \V(Y)} \mu(y \mid x) \log\frac{\mu(y\mid x)}{\bp(y\mid x)}\right) + \alpha_L\gamma \; \left(\sum_{y \in \V(Y)} \mu(y\mid x) \log \frac{1}{\mu(y\mid x)} \right) \right]  - \gamma  \H(\mu) \\ 
		&= \sum\alle 
			\Ex_{x \sim \mu_{\!_X}}  \left[ \sum_{y \in \V(Y)} \mu(y \mid x) \left(  \beta_L\; \log\frac{\mu(y\mid x)}{\bp(y\mid x)} + \alpha_L \gamma \; \log \frac{1}{\mu(y\mid x)} \right) \right]  - \gamma  \H(\mu) \\
		&= \sum\alle 
			\Ex_{x \sim \mu_{\!_X}}  \left[ \Ex_{y \sim \mu(Y \mid X=x)} \left(  \beta_L\; \log\frac{\mu(y\mid x)}{\bp(y\mid x)} + \alpha_L \gamma \; \log \frac{1}{\mu(y\mid x)} \right) \right]  - \gamma \sum_{\mat w \in \V(\dg M)} \mu(\mat w) \log \frac{1}{\mu(\mat w)} \\  
		&= \sum\alle 
			\Ex_{x,y \sim \mu_{\!_{XY}}}  \left[ \beta_L\; \log\frac{\mu(y\mid x)}{\bp(y\mid x)} + \alpha_L\gamma \; \log \frac{1}{\mu(y\mid x)}  \right]  - \gamma  \Ex_{\mat w \sim \mu} \left[ \log \frac{1}{\mu(\mat w)}\right] \\
		&= \Ex_{\mat w \sim \mu} \Bigg\{   \sum_{ X \xrightarrow{\!\!L} Y  } \left[
			\beta_L \log \frac{1}{\bp(y\mid x)}   - \beta_L  \log \frac{1}{\mu(y \mid x)}+ \alpha_L\gamma \log \frac{1}{\mu(y \mid x)} \right]\Bigg\}  -  \gamma  \Ex_{\mat w \sim \mu} \left[\log \frac{1}{\mu(\mat w)}\right] \\
		&=  \Ex_{\mat w \sim \mu} \Bigg\{ \sum_{ X \xrightarrow{\!\!L} Y  } \left[
			\beta_L \log \frac{1}{\bp(y\mid x)} +
	                        (\alpha_L\gamma - \beta_L ) \log
	                        \frac{1}{\mu(y \mid x)} \right] -
	                        \gamma \log \frac{1}{\mu(\mat w)}  \Bigg\}.  
	\end{align*}
\end{proof}

We can now prove         Proposition~\ref{prop:sem3}.
\restate{prop:sem3}{
If $\dg M$ is a PDG and $0 < \gamma \leq \min_L \nicefrac{\beta_L^{\dg M}}{\alpha_L^{\dg M}}$, then
$\bbr{\dg M}_\gamma^*$ is a singleton. 
}
\begin{proof}
It suffices to show that $\bbr{\dg
			  M}_\gamma$ is a strictly convex function of $\mu$,
since every strictly convex function has a unique minimum.
Note that
\begin{align*}
\bbr{M}_\gamma(\mu) 
	&= \Ex_{\mat w \sim \mu} \Bigg\{   \sum_{ X \xrightarrow{\!\!L} Y  } \left[
		\beta_L \log \frac{1}{\bp(y\mid x)} + ({\alpha_L}\gamma - \beta_L ) \log \frac{1}{\mu(y \mid x)} \right] - \gamma \log \frac{1}{\mu(\mat w)} \Bigg\} \\
	&= \Ex_{\mat w \sim \mu} \Bigg\{   \sum_{ X \xrightarrow{\!\!L} Y  } \left[ \gamma {\alpha_L} \log \frac{1}{\bp(y\mid x)} + 
		(\beta_L - {\alpha_L} \gamma) \log \frac{1}{\bp(y\mid x)} - (\beta_L - {\alpha_L} \gamma) \log \frac{1}{\mu(y \mid x)} \right] - \gamma \log \frac{1}{\mu(\mat w)} \Bigg\}  \\
	&= \Ex_{\mat w \sim \mu} \Bigg\{   \sum_{ X \xrightarrow{\!\!L} Y  } \left[ \gamma {\alpha_L} \log \frac{1}{\bp(y\mid x)} + 
		(\beta_L - {\alpha_L} \gamma) \log \frac{\mu(y\mid x)}{\bp(y\mid x)} \right] - \gamma \log \frac{1}{\mu(\mat w)} \Bigg\} \\
	&=  \sum_{ X \xrightarrow{\!\!L} Y  } \left[ \gamma {\alpha_L} \Ex_{x,y \sim \mu_{\!_{XY}}} \left[ \log \frac{1}{\bp(y\mid x)} \right] + 
		(\beta_L - {\alpha_L} \gamma) \Ex_{x\sim\mu_X}
          \kldiv[\Big]{\mu(Y\mid x)}{\bp( x)} \right] - \gamma \H(\mu). 
\end{align*}
	The first term, 
	\( \Ex_{x,y \sim \mu_{\!_{XY}}} \left[-\log {\bp(y\mid x)}\right] \) 
	is linear in $\mu$, as $\bp(y\mid x)$ does not depend on $\mu$. %
As for the second term, it is well-known that KL divergence is convex, in the sense that 
	\[ \kldiv{\lambda q_1 + (1-\lambda) q_2 }{ \lambda p_1 +
          (1-\lambda) p_2} \leq \lambda \kldiv {q_1}{ p_1} +
                (1-\lambda) \kldiv{q_2}{p_2}. \] 
	Therefore, for a distribution on $Y$, setting $p_1 =
 p_2 = \bp(x)$, for all conditional marginals $\mu_1(Y \mid X=x)$ and
			$\mu_2(Y\mid X=x)$,
	\[ \kldiv{\lambda \mu_1(Y\mid x) + (1-\lambda)
			  \mu_2(Y\mid x) }{ \bp(x) } \leq \lambda \kldiv
			   {\mu_1(Y\mid x)}{\bp(x)} + (1-\lambda)
								  \kldiv{\mu_2(Y\mid x)}{\bp(x)}. \] 
	So $\kldiv*{\mu(Y\mid x)}{\bp( x)}$ is convex. As
			convex combinations of convex functions are convex,
			the second term, $\Ex_{x\sim\mu_X}\kldiv*{\mu(Y\mid
			  x)}{\bp( x)}$, is convex.
Finally, negative entropy is well known to be strictly convex.                

			Any non-negative linear combinations of the three
			terms is convex, and if this combination applies a
			positive coefficient to the (strictly convex) negative entropy,
			it must be strictly convex. Therefore, as
			long as $\beta_L \geq \gamma$ for all edges $L \in
			\Ed^{\dg M}$, $\bbr{\dg M}_\gamma$ is
strictly convex.  The result follows.
\end{proof}

We next prove \Cref{prop:limit-uniq}.  The first step is provided by the
following lemma.
\begin{lemma}\label{lem:gamma2zero}
 $\lim\limits_{\gamma\to0}\bbr{\dg M}_\gamma^* \subseteq \bbr{\dg M}_0^*$. 
\end{lemma}
\begin{proof}
\def\lb{k}
\def\ub{K}  

Since $\IDef{\dg M}$ is a finite weighted sum of entropies
and conditional entropies over the variables $\N^{\dg M}$, which have
finite support%
, it is bounded.
Thus, there exist bounds $k$ and $K$ depending only on $\N^{\dg M}$ and
$\V^{\dg M}$, such that $\lb \leq \IDef{\dg M}(\mu) \leq \ub$ for all $\mu$.
Since $\bbr{\dg M}_\gamma = \Inc_{\dg M} + \gamma \IDef{\dg M}$,
it follows that, for all $\mu \in \V(\dg M)$, we have
\[ \Inc_{\dg M}( \mu) + \gamma\lb \leq~ \bbr{\dg M }_\gamma(\mu) 
\leq~  \Inc_{\dg M}( \mu) + \gamma\ub. \]
For a fixed $\gamma$, since this inequality holds for all $\mu$, and
both $\Inc$ and $\IDef{}$ are bounded below, it must be the case that  
\[
\min_{\mu \in \Delta\V(\dg M)} \Big[ \Inc_{\dg M}( \mu) + \gamma\lb \Big]
~\leq~ \min_{\mu \in \Delta\V(\dg M)}\bbr{\dg M }_\gamma(\mu) ~\leq~
\min_{\mu \in \Delta\V(\dg M)} \Big[ \Inc_{\dg M}( \mu) + \gamma\ub
    \Big], \] 
even though the distributions that minimize each expression will in general be different.
Let $\Inc(\dg M) = \min_{\mu} \Inc_{\dg M}(\mu)$.
Since $\Delta\V(\dg M)$ is compact, the minimum of the middle term is
achieved.  
Therefore, for $\mu_\gamma \in \bbr{\dg M}^*_\gamma(\mu)$ that
minimizes it, we have 
$$\Inc(\dg M) +\gamma \lb \le \bbr{\dg M }_\gamma(\mu_\gamma) \le
		 \Inc(\dg M) +\gamma \ub$$ for all $\gamma \ge 0.$
Now taking the limit as $\gamma\rightarrow 0$ from above, we get that
$\Inc(\dg M) = \bbr{\dg M }_0(\mu^*)$.
Thus, $\mu^* \in \bbr{\dg M}_0^*$, as desired.
\end{proof}

We now apply Lemma~\ref{lem:gamma2zero} to show that the limit as
$\gamma \to 
0$ is unique, as stated in \Cref{prop:limit-uniq}. 
\restate{prop:limit-uniq}{
	For all $\dg M$, $\lim_{\gamma\to0}\bbr{\dg M}_\gamma^*$ is a singleton.
}
\begin{proof}
First we show that $\lim_{\gamma \to 0}\bbr{\dg M}_\gamma^*$ cannot be empty.
Let $(\gamma_n) = \gamma_1, \gamma_2, \ldots$ be a sequence of
positive reals 
converging to zero.  For all $n$, choose some $\mu_n \in \bbr{\dg
M}_{\gamma_n}^*$. Because $\Delta\V(\dg M)$ is a compact metric
space, it is sequentially compact, and so, by the
Bolzano–Weierstrass Theorem, the sequence $(\mu_n)$ has at least one
accumulation point, say $\nu$. By our definition of the limit, $\nu \in
\lim_{\gamma\to0}\bbr{\dg M}_\gamma^*$, as witnessed by the sequence
$(\gamma_n, \mu_n)_n$.  It follows that $\lim_{\gamma\to0}\bbr{\dg
  M}_\gamma^* \ne \emptyset$.

Now, choose $\nu_1, \nu_2  \in  \lim_{\gamma\to0}\bbr{\dg
  M}_\gamma^*$. 
Thus, there are subsequences $(\mu_{i})$ and $(\mu_{j})$ of
$(\mu_n)$ converging
to $\nu_1$ and $\nu_2$, respectively.
By \Cref{lem:gamma2zero}, $\nu_1, \nu_2 \in \bbr{\dg M}_0^*$, so
$\Inc_{\dg M}(\nu_1) = \Inc_{\dg M}(\nu_2)$.  
Because  $(\mu_{j_n}) \to \nu_1$, $(\mu_{k_n}) \to \nu_2$, and
$\IDef{\dg M}$ is
continuous on $\Delta\V(\dg M)$,
we conclude that  
$(\IDef{\dg M}(\mu_{i}))\to \IDef{\dg M}(\nu_1)$ and
$(\IDef{\dg M}(\mu_{j}))\to \IDef{\dg M}(\nu_2)$.

Suppose that $\IDef{\dg
M}(\nu_1) \neq \IDef{\dg M}(\nu_2)$. Without loss of generality,
suppose that $\IDef{\dg M}(\nu_1) > \IDef{\dg M}(\nu_2)$. 
Since $(\IDef{\dg M}(\mu_{i})) \to \IDef{\dg M}(\nu_1)$, there exists some $i^*
\in \mathbb N$ such that for all $i > i^*$,  
$ \IDef{\dg M}(\mu_{i}) >  \IDef{\dg M}(\nu_2) $.
But then for all $\gamma$ and $i > i^*$, we have 
\[ \bbr{\dg M}_\gamma(\mu_i) = \Inc(\mu_i) + \gamma\IDef{\dg M}(\mu_i)
> \Inc(\nu_2)  
+ \gamma \IDef{\dg M}(\nu_2) = \bbr{\dg M}_\gamma(\nu_2),\]
contradicting the assumption that $\mu_{i}$ minimizes
$\bbr{\dg M}_{\gamma_{i}}$. We thus conclude that we
cannot have $\IDef{\dg M}(\nu_1) > \IDef{\dg M}(\nu_2)$.  By the same
argument, we also cannot have $\IDef{\dg M}(\nu_1) < \IDef{\dg
  M}(\nu_2)$, so $\IDef{\dg M}(\nu_1) =\IDef{\dg M}(\nu_2)$.  
  
Now, suppose that $\nu_1$ and $\nu_2$ distinct. Since $\bbr{\dg M}_\gamma$
is strictly convex for $\gamma > 0$, among the possible convex
combinations of $\nu_1$ and $\nu_2$, the distribution $\nu_3 = \lambda
\nu_1 + (1-\lambda) \nu_2$ that minimizes $\bbr{\dg M}_\gamma$ must
lie strictly between $\nu_1$ and $\nu_2$. 
Because $\Inc$ itself is convex and $\Inc_{\dg M}(\nu_1) = \Inc_{\dg
  M}(\nu_2) =: v$, we must have $\Inc_{\dg M}(\nu_3) \le v$. 
But since
$\nu_1,\nu_2 \in \bbr{\dg M}_0^*$ minimize $\Inc$,
we must have $\Inc_{\dg M}(\nu_3) \ge v$.
Thus, $\Inc_{\dg M}(\nu_3) = v$. 
Now, because, for all  $\gamma > 0$,
\[ \bbr{\dg M}_\gamma(\nu_3) = v + \gamma \IDef{\dg M}(\nu_3) 
 	< v + \gamma \IDef{\dg M}(\nu_1) = \bbr{\dg M}_\gamma(\nu_1), \] 
it must be the case that $\IDef{\dg M}(\nu_3) < \IDef{\dg M}(\nu_1)$. 
        
We can now get a contradiction by applying the same argument as that used to show
that $\IDef{\dg M}(\nu_1) =\IDef{\dg M}(\nu_2)$.  
    Because $(\mu_{i}) \to \nu_1$, there exists some
    $i^*$ such that for all $i > i^*$, we have $\IDef{\dg M}(\mu_{i}) >
    \IDef{\dg M}(\nu_3)$. Thus, for all $i > i^*$ and all
    $\gamma > 0$, 
    \[ \bbr{\dg M}_\gamma(\mu_{i}) = \Inc(\mu_{i}) + \gamma\IDef{\dg M}(\mu_{i}) > \Inc(\nu_3) 
    + \gamma \IDef{\dg M}(\nu_3) = \bbr{\dg M}_\gamma(\nu_3),\]
again contradicting the assumption that $\mu_{i}$ minimizes
$\bbr{\dg M}_{\gamma_{i}}$.
Thus, our supposition that $\nu_1$ was distinct from $\nu_2$ cannot hold, and so
$\lim_{\gamma \to 0}\bbr{\dg M}_\gamma^*$ must be a singleton, as desired.
\end{proof}

Finally, \Cref{prop:consist} is a simple corollary of \Cref{lem:gamma2zero} and \Cref{prop:limit-uniq}, as we now show. 
\restate{prop:consist}{
$\bbr{\dg M}^* \in \bbr{\dg M}_0^*$, so if $\dg M$ is consistent,
then $\bbr{\dg M}^* \in \SD{\dg  M}$.
}

\begin{proof}
By \Cref{prop:limit-uniq}, $\lim_{\gamma \to 0}\bbr{\dg M}_\gamma^*$
is a singleton. As in the body of the paper, we refer to its unique element by $\bbr{\dg M}^*$
\Cref{lem:gamma2zero} therefore immediately gives us $\bbr{\dg M}^* \in \bbr{\dg M}_0^*$.  

If $\dg M$ is consistent, then by \Cref{prop:sd-is-zeroset},
$\Inc({\dg M}) = 0$, so $\bbr{\dg M}_0(\bbr{\dg M}^*) = 0$, and thus
$\bbr{\dg M}^* 
\in \SD{\dg M}$. 
\end{proof}

	\subsection{PDGs as Bayesian Networks}
In this section, we prove Theorem~\ref{thm:bns-are-pdgs}.  
We start by recounting some standard results and notation, all of
which can be found in a standard introduction to information
theory (e.g., \cite[Chapter 1]{mackay2003information}).  

First, note that just as we introduced new variables to model joint dependence
in PDGs, we can view a finite collection $\mathcal X=X_1, \ldots, X_n$ of random
variables, where each $X_i$ has the same sample space, as itself a random
variable%
, taking the value $(x_1, \ldots, x_n)$ iff each $X_i$ takes the value $x_i$.
Doing so allows us to avoid cumbersome and ultimately irrelevant notation which treats sets of raomd variables differently, and requires lots of unnecessary braces, bold face, and uniqueness issues. 
Note the notational convention that the joint variable $X,Y$ may be indicated by a comma.

\begin{defn}[Conditional Independence]\label{defn:cond-indep}
    If $X$, $Y$, and $Z$ are random variables,
    and $\mu$ is a distribution over them, 
    then ${X}$ is \emph{conditionally independent of ${Z}$ given ${Y}$}, 
       (according to $\mu$),  denoted `${ X} \CI_\mu { Z}
        \mid { Y}$, iff for all ${ x}, { y}, { z} \in
        \V({X}, { Y},{ Z})$, we
        have $\mu({ x} \mid { y}) \mu({ z} \mid { y}) =
        \mu({ x,z} \mid { y})$.
\end{defn}

\begin{fact}[Entropy Chain Rule]\label{fact:entropy-chain-rule}
    If $X$ and $Y$ are random variables, then the entropy of the joint
   variable $(X,Y)$ can be written as $\H_\mu(X,Y) = 
\H_\mu( Y \mid X) + \H_\mu(X)$.
It follows that if $\mu$ is a
       distribution over the $n$ variables $X_1, \ldots, X_n$,  then
	\[ \H(\mu) = \sum_{i = 1}^n \H_\mu(X_i \mid X_1, \ldots X_{i-1}). \]
\end{fact}
\begin{defn}[Conditional Mutual Information]\label{defn:cmi}
   The \emph{conditional mutual information} between two (sets of) random
    variables is defined as  
    \[ \I_\mu(X ; Y \mid Z) := \sum_{x,y,z \in \V(X,Y,Z)} \mu(x,y,z)
        \log\frac{\mu(z) \mu(x,y,z)}{\mu(x,z)\mu(y,z)}. \] 
\end{defn}

\begin{fact}[Properties of Conditional Mutual Information]\label{fact:cmi}
For random variables $X,Y$, and $Z$ over a common set of outcomes,
distributed according to a distribution $\mu$,
the following properties hold:
\begin{enumerate}
    \item \textbf{(difference identity)} $\I_\mu(X ; Y \mid Z) =
                  \H_\mu(X \mid Y) - \H_\mu(X \mid Y, Z)$; 
   \item \textbf{(non-negativity)} $\I_\mu({ X }; { Y} \mid {Z}) \ge 0$;
    \item \textbf{(relation to independence)} $\I_\mu({ X }; { Y}
         \mid { Z}) = 0$ iff $X \CI_\mu Z \mid Y$.
\end{enumerate}
\end{fact}

We now provide the formal details of
the transformation of a BN 
into
a PDG.

	\begin{defn}[Transformation of a BN to a PDG]\label{def:bn2PDG}
Recall that a (quantitative) Bayesian Network $(G, f)$ consists of two
parts: its qualitative graphical structure $G$, 
described by a dag,
and its quantitative data $f$, an assignment of 
a cpd $p_i(X_i \mid \Pa(X_i))$ to each variable $X_i$.
If $\cal B$ is a Bayesian network on random variables
$X_1, \ldots, X_n$, we construct the corresponding PDG
$\PDGof{{\mathcal B}}$
			as follows: we take $\N := \{X_1, \ldots, X_n \} \cup
			\{ \Pa(X_1), \ldots, \Pa(X_n)\}$.  
That is, the variables of 
	  $\PDGof{{\mathcal B}}$
consist of all the variables in
${\cal B}$ together with a variable corresponding to the parents
of $X_i$%
.  (This will be used to deal with the hyperedges.) 
			The values $\V(X_i)$ for a random variable
			$X_i$ are unchanged, 
(i.e., $\V^{\PDGof{{\mathcal B}}}(\{X_i\}) := \V(X_i)$)
and $\V^{\PDGof{{\mathcal B}}}(\Pa(X_i)) := \prod_{Y \in \Pa(X_i)} \V(Y)$
(if $\Pa(X_i) = \emptyset$, so that $X_i$ has no parents, then we 
then we identify $\Pa(X_i)$ with $\pdgunit$ and
take $\V(\Pa(X_i)) = \{\star\}$). 
We take the set of edges $\Ed^{\PDGof{{\mathcal B}}} := \{ (\Pa(X_i), X_i) : 
i = 1, \ldots, n \} \cup \{ (\Pa_i, Y) : Y \in
			\Pa(X_i)\}$ to be the set of edges to a variable $X_i$
	  from its parents, together with an edge from
	  from $\Pa(X_i)$ to each of the elements of $\Pa(X_i)$, for
	  $i = 1, \ldots, n$.  
	Finally, we set $\mat p^{\PDGof{{\mathcal
				B}}}_{(\Pa(X_i), X_i)}$ to be the cpd associated
			with $X_i$ in $\cal B$, and for each node $X_j \in \Pa(X_i)$,
			we define
	\[ \mat p^{\PDGof{\mathcal B}}_{(\Pa(X_i),
			  X_j)}(\ldots, x_j, \ldots) = \delta_{x_j};\]
that is,
$\mat p_{(\Pa(X_i), X_j)}^{\PDGof{\mathcal B, \beta}}$ is the the cpd 
on $X_j$ that, given a setting $(\ldots, x_j, \ldots)$ of $\Pa(X_i)$, yields the distribution that puts all mass on $x_j$. 
\end{defn}

Let $\mathcal X$ be the variables of some BN $\mathcal B$, and
$\mathcal M = \pdgvars$ 
be the PDG $\PDGof{\mathcal B}$.
Because the set  $\mathcal N$ of variables in $\PDGof{{\mathcal
    B},\beta}$ includes  
variables of the form $\Pa(X_i)$, it is a strict superset of
$\mathcal X = \{X_1,\ldots, X_n\}$, the set of variables of $\mathcal B$.
For the purposes of this theorem, we identify a distribution
$\mu_{\mathcal X}$ over $\mathcal X$ 
with the unique distribution $\Pr_{\cal B}$ whose marginal on the
variables in $\mathcal X$ is $\mu_{\mathcal X}$ such that if $X_j \in
\Pa(X_i)$, then 
$\mu_{\mathcal N}(X_j = x_j' \mid \Pa(X_i) = (\ldots, x_j,\ldots)) =
1$ iff $x_j = x_j'$.  In the argument below, we abuse notation,
dropping the the subscripts $\mathcal X$ and $\mathcal N$ on a
distribution $\mu$.

\restate{thm:bns-are-pdgs}{
If $\cal B$ is a Bayesian network
and $\Pr_{\cal B}$ is the distribution it specifies, then
  for all $\gamma > 0$ and all vectors $\beta$ such
  that $\beta_L > 0$ for all edges $L$,
  $\bbr{\PDGof{\mathcal B, \beta}}_\gamma^* = \{ \Pr_{\cal B}\}$, 
and thus $\bbr{\PDGof{\mathcal B, \beta}}^* = \Pr_{\cal B}$.    
}
\begin{proof}
  For the cpd $p(X_i \mid \Pa(X_i))$ associated to a node $X_i$ in 
$\cal B$, we have that $\Pr_{\cal B}(X_i
\mid \Pa(X_i)) = p(X_i \mid \Pa(X_i))$.  
For all nodes $X_i$ in $\mathcal B$ and $X_j \in \Pa(X_i)$, 
by construcction, $\Pr_{\cal B}$, when viewed as a distribution on
$\mathcal N$, is also with the cpd on the edge from $\Pa(X_i)$ to
$X_j$.
Thus, $\Pr_{\cal B}$ is consistent with all the cpds in
$\PDGof{\mathcal B, \beta}$;
so$\Inc_{\PDGof{\mathcal B,\beta}}(\Pr_{\cal B}) = 0$.

We next want to show  that $\IDef{\PDGof{\mathcal B,\beta}}(\mu) \ge 0$ for all
distributions $\mu$.  To do this, we first need some definitions.
Let $\rho$ be a permutation of $1, \ldots,  n$.  Define an order
$\prec_{\rho}$ by taking $j \prec_{\rho} i$ if $j$ precedes $i$ in the
permutation; that is, if 
$\rho^{-1}(j)$ < $\rho^{-1}(i)$. Say that a permutation is \emph{compatible with
  $\mathcal B$} if $X_j \in \Pa(X_i)$ implies $j \prec_{\rho} i$.   There
is at least one permutation compatible with $\mathcal B$, since 
the graph underlying $\mathcal B$ is acyclic.
  
Consider an arbitrary distribution $\mu$ over the variables in
$\mathcal X$ (which we also view as a distribution over the variables
in $\mathcal N$, as discussed above).
Recall from \Cref{def:bn2PDG}
that the cpd on the edge in $\PDGof{{\cal B},\beta}$ from $\Pa(X_i)$ to $X_i$
is just the cpd associated with $X_i$ in ${\cal B}$, while the cpd on
the edge in $\PDGof{{\cal B},\beta}$ from $\Pa(X_i)$ to $X_j \in \Pa(X_i)$
consists only of deterministic distributions (i.e., ones that put
probability 1 on one element), which all have entropy 0.  
Thus,
\begin{equation}\label{eq:fact2}
\sum_{\ed LXY \in \Ed^{\PDGof{\mathcal B}}} \H_\mu(Y\mid
X)=\sum_{i=1}^n \H_\mu(X_i \mid \Pa(X_i)). 
\end{equation}

Given a permutation $\rho$, let ${\bf X}_{\prec_\rho i} = \{X_j: j
\prec_\rho i\}$.  Observe that 
\begin{align*}
    \IDef{\PDGof{\mathcal B,\beta}}(\mu)
 	&= \left[\sum_{\ed LXY \in \Ed^{\PDGof{\mathcal B}}} \H_\mu(Y\mid X) \right] - \H(\mu) \\
	&= \sum_{i=1}^n \H_\mu(X_i \mid \Pa(X_i)) - \sum_{i = 1}^n
\H_\mu(X_i \mid {\bf X}_{\prec_\rho i}) & \text{[by
    \Cref{fact:entropy-chain-rule} and \eqref{eq:fact2}]}\\ 
	&= \sum_{i=1}^n \Big[\H_\mu(X_i \mid \Pa(X_i)) - \H_\mu(X_i
  \mid {\bf X}_{\prec_\rho i} )\Big] \\ 
      &= \sum_{i=1}^n \I_\mu \Big( X_i ~;~ {\bf X}_{\prec_\rho i}
    \setminus \Pa(X_i) ~\Big|~ \Pa(X_i) \Big). & \text{[by
        \Cref{fact:cmi}]} 
\end{align*}

Using \Cref{fact:cmi}, it now follows that,
for all distributions $\mu$,
$\IDef{\PDGof{\mathcal B}}(\mu) \ge 0$.
Furthermore, for all $\mu$ and permutations $\rho$,
\begin{equation}\label{eq:key}
  \IDef{\PDGof{\mathcal B}}(\mu) = 0 \quad\mbox{ iff }\quad 
    \forall i.~X_i \CI_\mu {\bf X}_{\prec_\rho i}.
\end{equation}

Since the left-hand side of (\ref{eq:key}) is independent of $\rho$,
it follows that $X_i$ is independent of 
${\bf X}_{\prec_\rho i}$ for some permutation $\rho$ iff $X_i$ is independent of
  ${\bf X}_{\prec_\rho i}$ for every permutation $\rho$.  Since there
is a permutation compatible with $\mathcal B$, we get that 
$\IDef{\PDGof{\mathcal B,\beta}}(\Pr_{\cal B}) = 0$.
We have now shown that that $\IDef{\PDGof{\mathcal B, \beta}}$ and $\Inc$ are 
non-negative functions of $\mu$, and both are zero at $\Pr_{0\cal B}$. 
Thus, for all $\gamma \geq 0$ and all vectors $\beta$, we
have that   $\bbr{\PDGof{\mathcal B, \beta}}_\gamma( \Pr_{\cal
  B}) \le \bbr{\PDGof{\mathcal B, \beta}}_\gamma( \mu)$ for all
distributions $\mu$.  We complete the proof by showing that if
$\mu \ne \Pr_{\cal B}$, then 
$\bbr{\PDGof{\mathcal B, \beta}}_\gamma(\mu) > 0$
for $\gamma > 0$.

So suppose that $\mu \ne \Pr_{\cal B}$. 
Then $\mu$ must also match each cpd of $\cal B$,
for otherwise $\Inc_{\PDGof{\mathcal B,
\beta}}(\mu) > 0$, and we are done.  
Because $\Pr_{\cal B}$ is the \emph{unique} distribution that matches the 
both the cpds and independencies of $\cal B$, $\mu$ must not have all of the 
independencies of $\cal B$. 
Thus,
some variable $X_i$, $X_i$ is not independent of some nondescendant $X_j$ in
$\mathcal B$ with respect to $\mu$.  There must be some permutation
$\rho$ of the variables in $\mathcal X$ compatible with ${\mathcal B}$
such that $X_j \prec_{\rho} X_i$ (e.g., we can start with $X_j$ and
its ancestors, and then add the remaining variables appropriately).
Thus, it is not the case that $X_i$ is independent of $X_{\prec \rho,
  i}$, so by (\ref{eq:key}), $\IDef{\PDGof{\mathcal B}}(\mu) > 0$.
This completes the proof.
\end{proof}

\subsection{Factor Graph Proofs}
\cref{thm:fg-is-pdg,thm:pdg-is-fg} are immediate corolaries of their
more general counterparts, \cref{thm:pdg-is-wfg,thm:wfg-is-pdg}, which
we now prove.

\recall{thm:pdg-is-wfg}
\begin{proof}
	Let $\dg M := (\dg N, \mat v, \gamma \mat v)$ be the PDG in question.
	Explicitly, $\alpha^{\dg M}_L = v_L$ and $\beta_L^{\dg M} =  \gamma v_L$.
	By \Cref{prop:nice-score},
	\[ \bbr{\dg M}_\gamma(\mu)= \Ex_{\mat w \sim \mu} \Bigg\{   \sum_{ X \xrightarrow{\!\!L} Y  } \left[
		\beta_L \log \frac{1}{\bp(y\mid x)} + (
			\alpha_L
		\gamma - \beta_L ) \log \frac{1}{\mu(y \mid x)}
					\right] - \gamma \log \frac{1}{\mu(\mat w)}
			\Bigg\}.  \]
	Let $\{\phi_L\}_{L \in \Ed} := \Phi_{\dg N}$ denote the
			factors of the factor graph associated with $\dg M$. 
	Because we have $\alpha_L\gamma  = \beta_L$, the middle term cancels, leaving us with
	\begin{align*}
	\bbr{\dg M}_\gamma(\mu) &= \Ex_{\mat w \sim \mu} \Bigg\{   \sum_{ X \xrightarrow{\!\!L} Y  } \left[
		\beta_L \log \frac{1}{\bp(y\mid x)} \right] - \gamma \log \frac{1}{\mu(\mat w)} \Bigg\} \\
		&= \Ex_{\mat w \sim \mu} \Bigg\{   \sum_{ X \xrightarrow{\!\!L} Y  } \left[
	\gamma v_L \log \frac{1}{\phi(x,y)}  \right] - \gamma \log \frac{1}{\mu(\mat w)} \Bigg\} 
					&\text{[as $\beta_L = v_L \gamma$]}\\
		&= \gamma \Ex_{\mat w \sim \mu} \Bigg\{   \sum_{ X \xrightarrow{\!\!L} Y  } \left[
v_L \log \frac{1}{\phi(x,y)}
			 \right] -\log \frac{1}{\mu(\mat w)} \Bigg\} \\
			&= \gamma \GFE_{(\FGof{\dg N}, \mat v)}. 
	\end{align*}
	It immediately follows that the associated factor graph has 
	$\bbr{\dg M}^*_\gamma
 	= \{\Pr_{\Phi(\dg M)}\}$, because the free energy is clearly a constant plus the KL divergence from its associated probability distribution.
\end{proof}

\restate{thm:fg-is-pdg}{
	For all WFGs $\Psi = (\Phi,\theta)$ and all $\gamma > 0$,
	we have that
	$\GFE_\Psi
	= \nicefrac1{\gamma} \bbr{{\dg M}_{\Psi,\gamma}}_{\gamma} 
	+ C$   
	for some constant $C$, so
	$\Pr_{\Psi}$ is the unique element of
	$\bbr{{\dg M}_{\Psi,\gamma}}_{\gamma}^*$.
}
\begin{proof}
  In $\PDGof{\Psi,\gamma}$,  there is an edge $1 \to X_J$ for every $J
  \in \mathcal J$, and also edges 
  $X_J \tto X_j$ for each $X_j
    \in X_J$. Because the latter edges are deterministic, a
distribution $\mu$ that is not  consistent
with one of the edges, say $X_J \tto X_j$, has $\Inc_{\dg M}(\mu)
= \infty$.  This is a 
property of relative entropy: if there exist $j^* \in \V(X_j)$ and 
$\mat z^* \in \V(J)$ such that $\mat z^*_J \ne j^*$ and $\mu$ places positive
probability on their co-occurance (i.e., $\mu(j^*, \mat z^*) > 0$),
then we would have
\[ \Ex_{\mat z \sim \mu_{J}}\kldiv[\Big]{\mu(X_j \mid X_J = \mat z)}
	{\mathbbm1[X_j = \mat z_{j}]}
 	= \sum_{\substack{\mat z \in \V(X_J),\\ \iota \in \V(X_j)}} \mu(\mat z, \iota) \log \frac{\mu(\iota \mid \mat z)}{\mathbbm1[\mat z_j = \iota]}
	\geq \mu(\mat z^*, j^*) \log \frac{\mu(j^* \mid \mat z)}{\mathbbm1[\mat z^*_j = j_*]}
	= \infty. \]
Consequently, a distribution $\mu$ that does not satisfy the the projections has
$\bbr{\dg M_{\Psi,\gamma}}_\gamma(\mu) = \infty$ for every $\gamma$.
          Thus, a distribution that 
        has a finite score must match the constraints,
so we can identify such a distribution with its restriction to 
the original  
variables of $\Phi$.
Moreover, for all distributions $\mu$ with finite score and
projections $X_J \tto 
X_j$, the conditional entropy 
$\H(X_j \mid X_J) = -\Ex_\mu\log(\mu(x_j \mid x_J))$ and divergence from
the constraints are both zero. 
Therefore the per-edge terms for both $\IDef{\dg M}$
and $\Inc_{\dg M}$ can be safely ignored for the projections.
Let $\bp[J]$ be
the normalized distribution $\frac{1}{Z_J}\phi_J$ over $X_J$,
where $Z_J = \sum_{x_J} \phi_J(x_J)$ is the appropriate normalization constant.
By
\cref{def:wfg2pdg}, we have $\PDGof{\Psi,\gamma} = (\UPDGof{\Phi}, \theta, \gamma\theta)$,
so by \cref{prop:nice-score},
	\begin{align*}
\bbr{\PDGof{\Psi,\gamma}}_\gamma(\mu) 
	&= \Ex_{\mat x \sim \mu} \Bigg\{   \sum_{ J \in \mathcal J } \left[
		\beta_J \log \frac{1}{ \bp[J](x_J) } + 
			(\alpha_J \gamma -\beta_J)
		 \log \frac{1}{\mu(x_J)} \right] - \gamma \log \frac{1}{\mu(\mat x)} \Bigg\} \\
		 &= \Ex_{ \mat x \sim \mu} \Bigg\{   \sum_{ J \in \mathcal J } \left[
	 		(\gamma\theta_J) \log \frac{1}{ \bp[J](x_J) } + 
	 			(\theta_J \gamma - \gamma\theta_J)
	 		 \log \frac{1}{\mu(x_J)} \right] - \gamma \log \frac{1}{\mu(\mat x)} \Bigg\} \\
		&= \Ex_{ \mat x \sim \mu} \Bigg\{  \sum_{ J \in \mathcal J }\left[
			\gamma\theta_J \log \frac{1}{\bp[J](x_J)}  \right] - \gamma \log \frac{1}{\mu(\mat x)} \Bigg\} 
			\\
		&= \gamma \cdot \Ex_{\mat x \sim \mu} \Bigg\{  \sum_{ J \in \mathcal J } \theta_J
			\log \frac{Z_J}{\phi_J(x_J)}   -\log \frac{1}{\mu(\mat x)} \Bigg\} \\
		&= \gamma \cdot \Ex_{\mat x \sim \mu} \Bigg\{  \sum_{ J \in \mathcal J } \theta_J \left[
			\log \frac{1}{\phi_J(x_J)} + \log Z_J \right]  - \log \frac{1}{\mu(\mat x)} \Bigg\} \\
		&= \gamma \cdot \Ex_{\mat x \sim \mu} \Bigg\{  \sum_{ J \in \mathcal J } \theta_J 
			\log \frac{1}{\phi_J(x_J)}  - \log \frac{1}{\mu(\mat x)} \Bigg\}
			 +  \sum_{J \in \mathcal J} \theta_J \log Z_J  \\
        	&= \gamma\, \GFE_{\Psi} + k \log \prod_{J} Z_J,
	\end{align*}
which differs from $\GFE_{\Psi}$ by the value $\sum_J \theta_J \log Z_J$, which 
is constant in $\mu$.

\end{proof}
\fi
\end{document}